\documentclass{article} 
\usepackage{iclr2026_conference,times}

\usepackage{amsmath,amsfonts,bm}

\def\eqref#1{equation~\ref{#1}}

\def\1{\bm{1}}

\DeclareMathAlphabet{\mathsfit}{\encodingdefault}{\sfdefault}{m}{sl}
\SetMathAlphabet{\mathsfit}{bold}{\encodingdefault}{\sfdefault}{bx}{n}

\DeclareMathOperator*{\argmin}{arg\,min}
\DeclareMathOperator*{\Argmax}{Arg\,max}

\usepackage{hyperref}
\usepackage{url}
\usepackage{bm}

\usepackage[utf8]{inputenc} 
\usepackage[T1]{fontenc}
\usepackage[pdftex]{graphicx}
\usepackage{hyperref}       
\usepackage{url}            
\usepackage{booktabs}       
\usepackage{amsthm}		
\usepackage{mathtools}
\usepackage[noend]{algorithmic}    
\usepackage{algorithm}    
\usepackage{subcaption}
\usepackage{wrapfig}

\newtheorem{theorem}{Theorem}[section]
\newtheorem{proposition}{Proposition}[section]
\newtheorem{lemma}{Lemma}[section]
\newtheorem{remark}{Remark}[section]
\usepackage[most]{tcolorbox}

\newtcolorbox{remarkbox}{
    colback=gray!20,    
    colframe=black,     
    boxrule=0.5mm,      
    sharp corners,      
    left=2mm,           
    right=2mm,          
    top=2mm,            
    bottom=2mm          
}

\title{LoRA meets Riemannion: Muon Optimizer for Parametrization-independent Low-Rank Adapters}

\author{Vladimir Bogachev\\
  HSE University, \\Lomonosov Moscow State University\\
  \texttt{vabogachev@hse.ru}
  \And
  Vladimir Aletov \\
  HSE University\\
  \And
  Alexander Molozhavenko \\
  HSE University\\
  \And
  Denis Bobkov \\
  HSE University\\
  \And
  Vera Soboleva\\
  HSE University\\
  \And Aibek Alanov \\
  HSE University\\
  \And Maxim Rakhuba\\
  HSE University\\
}

\iclrfinalcopy 
\begin{document}

\maketitle

\begin{abstract}
This work presents a novel, fully Riemannian framework for Low-Rank Adaptation (LoRA) that geometrically treats low-rank adapters by optimizing them directly on the fixed-rank manifold. This formulation eliminates the parametrization ambiguity present in standard Euclidean optimizers. Our framework integrates three key components to achieve this: (1) we derive Riemannion, a new Riemannian optimizer on the fixed-rank matrix manifold that generalizes the recently proposed Muon optimizer; (2) we develop a Riemannian gradient-informed LoRA initialization, and (3) we provide an efficient implementation without prominent overhead that uses automatic differentiation to compute arising geometric operations while adhering to best practices in numerical linear algebra. Comprehensive experimental results on both LLM and diffusion model architectures demonstrate that our approach yields consistent and noticeable improvements in convergence speed and final task performance over both standard LoRA and its state-of-the-art modifications.
\end{abstract}

\section{Introduction}%
\label{sec:introduction}

Large language models (LLMs) have demonstrated remarkable capabilities 
	across a wide range of natural language processing tasks \citet{brown2020language,touvron2023open,touvron2019open}. However, the 
	computational and storage costs associated with training and deploying 
	such models at scale pose significant challenges. To reduce these costs, 
	parameter-efficient fine-tuning techniques such as low-rank adaptation 
	(LoRA) \citet{hu2022lora} have emerged as a practical solution. LoRA 
	enables efficient adaptation of pre-trained models by embedding learnable 
	low-rank matrices into specific weight updates, allowing most of the original 
	parameters to remain frozen.
    In particular, the main idea of LoRA is to fine-tune a pretrained 
	model using a rank-$r$ correction matrix $\Delta W$: 
	\begin{equation*}
		W + \Delta W = W + AB^\top, 
		\quad A \in \mathbb{R}^{m \times r}, 
		\quad B \in \mathbb{R}^{n \times r},
	\end{equation*}
	where $W$ remains constant during training and $A, B$ are optimized via 
	gradient-based optimization methods.

Despite its efficiency, the dominant practice of optimizing the LoRA factors $(A,B)$ with Euclidean optimizers such as SGD \citep{robbins1951stochastic}, Adam \citep{kingma2014adam}, 
	Adagrad \citep{duchi2011adaptive}, RMSProp \citep{tieleman2012lecture}, etc. that misaligned with the geometry of the low-rank constraint. The same update $\Delta W$ can be represented by infinitely many factorizations: for any $A\in\mathbb{R}^{m\times r}$, $B\in\mathbb{R}^{n\times r}$ and any invertible matrix $S\in\mathbb{R}^{r\times r}$, we may write:
    \begin{equation} \label{eq:intro_skeleton_parametrization}
        \Delta W = A B^\top = \widetilde A {\widetilde B}^\top, \quad \text{where}\quad \widetilde A  = A S,\quad \widetilde B = B S^{-\top}.
    \end{equation}
Ideally, training should be \emph{reparameterization (transformation) invariant}: the update to $\Delta W$ must not depend on which factorization $(A,B)$ is used \citep{yen2024lora}.  Empirically, this lack of invariance manifests as unbalanced learning where one factor dominates and the other stalls, fragile hyperparameter sensitivity, and path-dependent solutions.
These issues have prompted geometry-aware formulations. Riemannian treatments of low-rank models operate on the \emph{fixed-rank manifold} rather than the ambient factor space, projecting gradients to the tangent space and retracting back to the manifold. Such steps can be implemented efficiently when $r \ll \min\{m,n\}$ and avoid forming full-size matrices. 
Within the LoRA literature, existing Riemannian approaches either use standard SGD-type optimizers \citet[LORO]{moparameter} or rely on the Adam \citep{zhang2024riemannian} optimizer for auxiliary matrices within the chosen parameterization, which deviate them from the Riemannian framework and introduce dependence on parameterization.

In this work, we introduce a \emph{fully Riemannian} framework for training LoRA that optimizes the adapter $X \!=\! \Delta W$ directly on the fixed-rank matrix manifold
\[
\mathcal{M}_r \;=\; \{ X \in \mathbb{R}^{m \times n} : \mathrm{rank}(X)=r \},
\]
eliminating factorization ambiguity by construction. Central to our approach is \emph{Riemannion}, a new Riemannian optimizer on $\mathcal{M}_r$ that \emph{generalizes the recently proposed Muon optimizer~\citep{jordan2024muon}} to the fixed-rank setting. In contrast to prior Riemannian LoRA variants that port Adam-like mechanics to the manifold with ad hoc choices, our design inherits Muon’s geometry-aligned normalization, yielding transformation invariance of the learned update. We further propose a Riemannian gradient–informed initialization that places the initial adapter at a good location on $\mathcal{M}_r$, and we provide a practical, low-overhead implementation that assembles projections, retractions, and vector transports via automatic differentiation, also following best practices from numerical linear algebra. 
Extensive experiments on LLM and diffusion architectures show consistent gains in convergence speed and final task performance over standard LoRA and recent state-of-the-art modifications.

Our contributions are as follows:
\begin{itemize}
    \item \textbf{Riemannion: Muon on the fixed-rank manifold.} We derive \emph{Riemannion}, the first optimizer that \emph{generalizes Muon} to the manifold $\mathcal{M}_r$ of fixed rank matrices. 
    \item \textbf{Riemannian gradient-informed initialization.} We propose an initialization strategy which yields best alignment between the initial Riemannian gradient and the Euclidean gradient. We also propose an efficient way for this strategy by using a randomized SVD algorithm with implicit matrix multiplication (Section~\ref{sec:locally_optimal_initialization}). Finally, we show the connection of this initialization to LoRA-GA.
    \item \textbf{Efficient implementation with automatic differentiation.} We pay special attention to numerical implementation to make the method robust without any prominent overhead compared to vanilla LoRA at small ranks. 
    \item \textbf{Comprehensive empirical validation.} We showcase the performance of our framework for fine-tuning LLMs and in subject-driven
generation using diffusion models. Among positive effects that we observe are: boost in
target metrics, improved convergence, and reduction of variance.
\end{itemize}

\section{Related Work}%
\label{sec:related_work}

	The problem of an optimal initial guess selection for low-rank LLM 
	adaptation has been addressed 
	in a sequence of works: the authors \citet[PiSSA]{meng2024pissa} have suggested 
	a heuristic that involves using a low-rank truncated SVD 
	of pretrained parameters as an initial point 
	for LoRA and its orthogonal complement as frozen layer's parameters, so that 
	the tuning process starts without changing the starting value of the loss function. 
	A similar approach was implemented by \citet[MiLoRA]{wang2024milora}
	with the main difference of optimizing the smallest singular components
	of unadapted parameter matrix.
    A context-aware initialization was considered in~\citep[CorDA]{yang2024corda} and \citep[COALA]{parkina2025coalanumericallystableefficient} proposes a numerically robust inversion-free framework 
    for low-rank weighted approximations for this setting.
	Another idea is 
	to initialize LoRA with a subset of left and right 
	singular vectors of a doubled-rank truncated SVD of the loss function gradient at the starting 
	parameters, proposed by \citet[LoRA-GA]{wang2024lora}.
    We show direct connection of this method to our Riemannian initialization strategy and propose how to additionally significantly accelerate the  computation of SVD using our approach. 
	Attempting to overcome the asymmetry in the initialization of vanilla LoRA fine-tuning 
	process, \citep[LoRA+]{hayou2024lora+} introduced a scale-free step size 
	selection for LoRA factors.

	Riemannian optimization is widely used for  algorithms on
	matrix manifolds and allows for exploiting task geometry or 
	imposing additional constraints. For example, a Riemannian solution for the  
	extreme eigenpairs search problem was described in \citet{absil2009optimization,baker2008riemannian},
	a matrix completion task, which is common in collaborative  
	filtering for recommender systems,  
	via optimization on the fixed-rank manifold~\citep{vandereycken2013low}, a Riemannian approach on 
	the manifold of matrices with orthonormal columns 
	(the Stiefel manifold) was used by \citet{wisdom2016full} for 
	diminishing the problem of vanishing and exploding gradients 
	in recurrent neural networks, etc. The book 
	\citet{trendafilov2021multivariate} also presents a comprehensive 
	description of useful manifolds for solutions of the data science 
	problems. For the deeper understanding of  applied differential geometry techniques
	see the books \cite{absil2009optimization} 
	and \cite{boumal2023introduction}.

	The idea of using Riemannian optimization has recently 
	started to emerge for the large language models. 
    For example, 
	the fine-tuning of LLMs with the help of the Stiefel manifold was considered
	in the work \citet{hu2024retraction}. The authors of \citep{zhang2024riemannian} 
	introduced the Riemannian inspired modification of Adam.
	The authors of \citet[LORO]{moparameter}  applied the Riemannian
	optimization techniques for pretraining LLMs on the fixed-rank 
	manifold. Parametrization that is used in our work can potentially help in this setting  as well, by additionally avoiding potential overheads and instabilities, arising due the explicit inversion of Gram matrices.

\section{Preliminaries}

\subsection{Muon optimizer}
\label{sec:muon}

\emph{Muon} is an optimizer designed specifically for matrix‑valued parameters in a network’s hidden layers. Empirically, it accelerates training on language and vision workloads while leaving scalar/vector parameters and the input/output layers to a conventional optimizer such as AdamW. At a high level, Muon takes the step that stochastic gradient descent with momentum (SGDM) would make on a weight matrix and \emph{orthogonalizes} that update before applying it. Orthogonalization acts as a per‑layer, per‑step preconditioner that equalizes singular values of the update, which mitigates the collapse of updates into a few dominant directions~\citep{jordan2024muon}.
 More specifically, let $W\in\mathbb{R}^{n\times m}$ be a hidden‑layer weight. With gradient $G_t=\nabla_W\mathcal{L}(W_t)$ and momentum $M_t=\beta M_{t-1}+G_t$, Muon computes
\[
\widetilde M_t \;\approx\; \mathrm{Ortho}(M_t)
\quad\text{and}\quad
W_{t+1}=W_t-\eta\,\widetilde M_t,
\]
where $\mathrm{Ortho}(\cdot)$ denotes the nearest semi‑orthogonal matrix in Frobenius norm, i.e.,
\begin{equation} \label{eq:ortho-base}
\mathrm{Ortho}(G)\;=\;\arg\max_{O}\bigl\{\|O-G\|_F:\ O^\top O=I\ \text{or}\ OO^\top=I\bigr\}.
\end{equation}
Computing $\mathrm{Ortho}(G)$ exactly amounts to taking the SVD $G=USV^\top$ and returning $UV^\top$, which is too slow to do at every iteration. Muon instead applies a Newton–Schulz (NS) iteration that—after normalizing $G$ --- implements a composition of a fixed low‑degree polynomial in $GG^\top$ acting on $G$ and converges to $UV^\top$. 
Leveraging efficient matrix multiplication operations results in a highly performant iteration.
We will write $\widetilde M_t = NS(M_t)$ for short.

\paragraph{LMO interpretation.} Muon’s step admits a clean linear minimization oracle (LMO) interpretation~\citep{bernstein2025deriving}. 
Indeed, consider the operator-norm unit ball $\mathcal{B}_{2}=\{X:\|X\|_{2}\le 1\}$. The linear minimization oracle (LMO) over $\mathcal{B}_{2}$ at matrix $M_t$ given by its SVD $M_t=USV^\top$  is
\begin{equation} \label{eq:lmo_base}
UV^\top\in\Argmax_{\|S\|_{2}\le 1}\ \langle M_t,S\rangle.
\end{equation}

\paragraph{Applying Muon to LoRA.} In LoRA, a frozen weight $W_0\in\mathbb{R}^{n\times m}$ is adapted via a low‑rank update $W=W_0+\alpha\,B A$ with $B\in\mathbb{R}^{n\times r}$, $A\in\mathbb{R}^{r\times m}$ and small $r$. Each trainable factor ($A$ and $B$) is a 2D parameter, so Muon can be applied \emph{per factor}:
\[
\begin{split}
&M^{(A)}_t\gets \beta M^{(A)}_{t-1} + \,\nabla_A\mathcal{L},\quad
A_{t+1}\gets A_t-\eta_A\,\mathrm{NS}\bigl(M^{(A)}_t\bigr),\\
&M^{(B)}_t\gets \beta M^{(B)}_{t-1}+\nabla_B\mathcal{L},\quad
B_{t+1}\gets B_t-\eta_B\,\mathrm{NS} \bigl(M^{(B)}_t\bigr).
\end{split}
\]
Note that acting on the two factors separately makes Muon non–reparameterization‑invariant: its per‑factor orthogonalization depends on arbitrary scalings or rotations, skewing the weight‑space step and often letting one factor dominate.

\subsection{Riemannian optimization}
\label{sec:riemann}

Let $\mathcal{M}_r=\{X\in\mathbb{R}^{m\times n}:\operatorname{rank}(X)=r\}\subseteq\mathbb{R}^{m\times n}$
	be a smooth manifold of fixed-rank matrices~\cite[Example 8.14]{lee2003smooth}.
    Let every point $X$ of $\mathcal{M}_r$ be equipped with a \emph{tangent plane} $\mathcal{T}_X\mathcal{M}_r$.
Thinking geometrically, the tangent plane  plays the role of the best local, flat approximation to this curved set: if you “zoom in’’ at $X$, the manifold looks like a plane.
We will now discuss how to numerically parametrize points on a manifold and its tangent plane.
    Every rank-$r$ matrix $X\in \mathcal{M}_r$ can be represented using  
	matrices $A_L \in \mathbb{R}^{m \times r}, 
	B_r\in \mathbb{R}^{n \times r}$ with orthonormal columns and a square matrix $G\in\mathbb{R}^{r\times r}$  as
	\begin{equation} \label{eq:point}
		X = A_L G B_R^\top.
	\end{equation}
    For example, one may think of a thin SVD, in which case $G$ becomes the diagonal matrix of singular values, but other representations are also possible and will be convenient for our purposes.
	Using~(\ref{eq:point}), any tangent vector  $\xi \in \mathcal{T}_X\mathcal{M}_r$ can be represented in a matrix factorization format:
	\begin{equation}\label{eq:skeleton_decomp}
		\xi = \begin{bmatrix}\dot{A} &  A_L \end{bmatrix}
		\begin{bmatrix}  B_R & \dot{B} \end{bmatrix}^\top, \quad \dot{A}^\top A_L = 0,
	\end{equation}
    so that any tangent vector can be identified in terms of the tuple: $(\dot{A}, \dot{B})$.
    Since the factor matrices contain $2r$ columns, we immediately have that $\mathrm{rank}\,\xi \le 2r$.
	Another remarkable fact is that the point $X \in \mathcal{M}_r$ itself lies in the $\mathcal{T}_{X}\mathcal{M}_r$
	with $\dot{A} = 0, \dot{B} = B$.
	Given a matrix $Z \in \mathbb{R}^{m \times n}$, its 
	orthogonal projection $P_{\mathcal{T}_X\mathcal{M}_r}\, Z$ onto 
	the tangent space  $\mathcal{T}_X\mathcal{M}_r$ 
	(with the parameterization given in (\ref{eq:skeleton_decomp})) 
	can be computed as follows:
    \begin{equation}\label{eq:proj_fixed_rank}
        P_{T_X \mathcal{M}_r}(Z)=A_LA_L^\top Z \; + \; \left(I - A_L A_L^\top \right) ZB_RB_R^\top.
    \end{equation}
    and, hence, can be represented in the form of (\ref{eq:skeleton_decomp}) with $\dot{A} = (I - A_L A_L^\top)ZB_R$, and $\dot{B} = Z^\top A_L$.

Let $F \colon 
	\mathbb{R}^{m \times n} \to \mathbb{R}$ be a differentiable function 
	with the Euclidean gradient $\nabla F \in \mathbb{R}^{m \times n}$.
	Within the Riemannian optimization framework, we solve the following task 
     optimization problem:
	\begin{equation*}
		\min_{X \in \mathcal{M}_r} F(X).
	\end{equation*}
    When constructing the algorithms, we need to work with three key objects: Riemannian gradient, retraction and vector transport. 
	The Euclidean gradient is a direction of the steepest local 
	increase~$F$.
	Therefore, it is common to use the Riemannian gradient 
	--- the direction of the steepest
	local increase of corresponding smooth function value along 
	the manifold, which lies in the tangent space \citep[chap.~3.6]{absil2009optimization}. Given 
	the Euclidean gradient $\nabla F$,  one may 
	endow the tangent space $\mathcal{T}_X \mathcal{M}_r$ with a natural scalar product and 
	derive a formula for the direction of the local steepest ascent of $F$ with respect to the manifold. 
	This unique direction is called the Riemannian gradient and can be computed as follows:
	\begin{equation}\label{eq:riem_grad}
		\mathrm{grad}\,F(X) = P_{\mathcal{T}_X \mathcal{M}_r}\,(\nabla F(X)),
		\quad X \in \mathcal{M}_r.
	\end{equation}
A simple and robust retraction that maps a tangent step $\xi$ (for example, $\xi$ is a negative Riemannian gradient) back to the manifold is the truncated SVD:
\begin{equation}\label{eq:retraction}
R_X(\xi)\equiv R(X + \xi) =\operatorname{SVD}_r(X+\xi),
\end{equation}
i.e., the best rank‑$r$ approximation of $X+\xi$ in Frobenius norm.
Note that here we do not need to compute the full SVD and can utilize low-rank structure of $X$ and $\xi$, leading to $\mathcal{O}((m+n)r^2 + r^3)$ operations~\citep{absil2015low}. 
Finally, because tangent spaces change from an optimization step to step, momentum (an accumulated tangent vector) must be moved between them via a \emph{vector transport}. For embedded manifolds like $\mathcal{M}_r$, a standard choice is the \emph{projection transport}
\begin{equation}\label{eq:transport}
\mathcal{T}_{Y\to X}(\xi)\;=\;P_{T_X \mathcal{M}_r}(\xi),\qquad \xi\in\mathcal{T}_Y\mathcal{M}_r,
\end{equation}
which simply reprojects the same ambient matrix $\xi$ onto the new tangent space at $X$.

\section{Riemannion}%
\label{sec:riemannion}

In this section, we focus on the setting of parameter-efficient fine-tuning and, hence, to the fixed-rank manifold. 
For fine-tuning of one layer the optimization problem becomes:
\[
    \mathcal{L}(W + \Delta W) \to \min_{\Delta W\in\mathcal{M}_r}, 
\]
where $\mathcal{L}$ is a differentiable loss function.
Note that optimizing \emph{on} $\mathcal{M}_r$ removes the ambiguity of factorized parameterizations, because all computations are carried out in the intrinsic space of the product $X$ rather than in any particular factorization.
So the formulas we write below will naturally be reparameterization-invariant.
Let $G_t=P_{T_W \mathcal{M}_r}(\nabla \mathcal{L}(W_t))$ be the Riemannian gradient. A Riemannian \emph{heavy‑ball} \citep{polyak1964some} momentum step reads
\begin{align}
&M_t \;=\; \beta\,\widehat M_{t-1} \;+\; G_t, \quad \widehat M_{t-1} \;=\; \mathcal{T}_{W_{t-1}\to W_t}(M_{t-1}),  \label{eq:momentum_accum}\\
&\Delta W_{t+1} \;=\; R\left(\Delta W_t-\eta\,M_t\right), \label{eq:update}
\end{align}
with $M_0=0$, momentum parameter $\beta\in[0,1)$, and stepsize $\eta>0$. 

Let us now discuss how to introduce a Muon-like variant of this iteration, which we refer to as Riemannion.
A direct projection onto the set of orthogonal matrices $\mathrm{Ortho}(M_t)$ presents two challenges.
First, such a step does not respect the underlying Riemannian geometry.
To address this, we propose to find the best approximation of $\mathrm{Ortho}(M_t)$ on the tangent plane $T_W \mathcal{M}_r$.
Such a solution is given via the projection onto the tangent plane $P_{T_W \mathcal{M}_r} (\mathrm{Ortho}(M_t))$. 
However, a second issue arises: although this projection is low-rank, its computation remains inefficient because the input matrix is of full rank.
Let us notice that $M_t\in T_{\Delta W_t} \mathcal{M}_r$ and, hence, is of rank at most $2r$ (Section~\ref{sec:riemann}).
At the same time, the LMO interpretation (Section~\ref{sec:muon}) provides several admissible low-rank solutions, including one where $\mathrm{Ortho}r(\cdot)$ replaces only the first $2r$ singular values with $1$, while all others are set to $0$.
Consequently, we obtain the following update rule:
 \begin{equation}\label{eq:lmo-riem-sol}
\widetilde{M}_t = P_{T_{\Delta W_t} \mathcal{M}_r}(\mathrm{Ortho}_r(M_t)).
\end{equation}
Note that $\mathrm{Ortho}_r(\cdot)$ exactly preserves the column and row spaces of $M_t \in T_{\Delta W_t} \mathcal{M}_r$.  
Although $P_{T_{\Delta W_t} \mathcal{M}_r}(\mathrm{Ortho}_r(M_t))$ does not yield singular values exactly equal to $1$, in practice they remain in a close proximity.  
In particular, in our experiments they were always in the interval $(0.9, 1.1)$.
This behavior is reminiscent of the Newton--Schulz iteration, which likewise produces approximate singular values.  
To eliminate this inexactness and obtain an accurate solution in the intersection of $T_{\Delta W_t}\mathcal{M}_r$ with the set of matrices whose first $2r$ singular values are exactly $1$, one could apply various strategies. For example, one may formally apply the alternating projection method~\citep[Theorem 4, Section 2]{cheney1959proximity,boyd2003}:
\begin{equation}\label{eq:riem-altern-proj}
\widetilde{M}_t = P_{T_{\Delta W_t} \mathcal{M}_r}(\mathrm{Ortho}_r(\dots P_{T_{\Delta W_t} \mathcal{M}_r}(\mathrm{Ortho}_r(M_t)))). 
\end{equation}
As an alternative, we could solve the LMO problem on the tangent plane:
\begin{equation}\label{eq:riem-lmo-tangent}
\widetilde{M}_t  \in \Argmax_{\substack{S\in T_{\Delta W_t}\mathcal{M}_r \\ \|S\|_{2}\le 1}}\ \langle M_t,S\rangle,
\end{equation}
where we proposed to apply observations from~\citep{cesista2025sdnr} to the fixed-rank manifold.
In Appendix~\ref{sec:lmo_bounds}, we show that~(\ref{eq:lmo-riem-sol}) is an approximate solution to the maximization problem~(\ref{eq:riem-lmo-tangent}).
Applying~(\ref{eq:riem-altern-proj}) or accurately solving~(\ref{eq:riem-lmo-tangent}) on the intersection of two convex sets using convex optimization methods is more computationally expensive, and our experiments indicate that it has little to no impact on the overall convergence of the optimizer.

\begin{algorithm}[h!tp]
		\caption{\texttt{OrthoLR} (efficient computation of $\mathrm{Ortho}_r(\xi)$ for $\xi\in T_{X}\mathcal{M}_r$).}
		\label{alg:fr_procrustes}
		\textbf{Require:}
        $\xi \in T_{X}\mathcal{M}_r$ given by $(\dot{A}, \dot{B})$ from~(\ref{eq:skeleton_decomp}); $A_L, B_R$ such that $X=A_LG B_R^\top$ as in~(\ref{eq:point})\\
		\textbf{Ensure:} $A\in\mathbb{R}^{m \times 2r}, B\in\mathbb{R}^{m \times 2r}$: $AB^T = \mathrm{Ortho}(\xi)$.
		\begin{algorithmic}[1]
            \STATE $Q_L, T_L = \mathtt{qr}([A_L, \dot A ])$, $\quad Q_R, T_r = \mathtt{qr}([\dot B, B_R ]^\top)$.
			\hfill\COMMENT{$\mathcal{O}\left((m+n) r^2\right)$}

            \STATE $U_L, \_, V_r^\top = \mathtt{SVD}\left(T_L \, T_R^\top\right)$.
			\hfill\COMMENT{$\mathcal{O}\left(r^3\right)$}

            \STATE $\dot A = Q_L U_L,\quad  \dot B = Q_R V_R$.
			\hfill\COMMENT{$\mathcal{O}\left((m+n) r^2\right)$}

		\end{algorithmic}
	\end{algorithm}

\begin{algorithm}[h!tp]
		\caption{\texttt{ProjectLR} (efficient computation of $P_{T_{X}\mathcal{M}_r}(Z)$ for a rank-$r'$ matrix $Z$).}
		\label{alg:transport}
		\textbf{Require:} 
    		$A \in \mathbb{R}^{m \times r'}, B \in \mathbb{R}^{n\times r'}$ such that $Z = AB^\top$;
			$A_L, B_R$ such that $X=A_LG B_R^\top$ as in~(\ref{eq:point}).\\
		\textbf{Ensure:}
			$\xi = P_{T_{X}\mathcal{M}_r} (Z)$ given by $(\dot{A}, \dot{B})$ from~(\ref{eq:skeleton_decomp}).
		\begin{algorithmic}[1]
			 \STATE $\dot{A} \coloneqq (A - A_R (A_R^\top A))(B^\top B_R), \quad  \dot{B} \coloneqq B(A^\top A_L)$ 
			\hfill\COMMENT{$\mathcal{O}((m + n){r'}^2)$}
		\end{algorithmic}
	\end{algorithm}

Let us finally show that $\widetilde{M}_t$ from~(\ref{eq:lmo-riem-sol}) can be computed efficiently using $\mathcal{O}((m+n)r^2 + r^3)$ arithmetic operations.
Indeed, first of all we need to apply $\mathrm{Ortho}_r(\cdot)$ to a tangent vector $M_t$.
From~(\ref{eq:proj_fixed_rank}), we know that a tangent vector can be represented in a form of a rank-$2r$ matrix.
Such a representation can always be transformed into the compact SVD form with 2 QR decompositions and a single full SVD of a $2r\times 2r$ matrix, see Algorithm~\ref{alg:fr_procrustes} called \texttt{OrthoLR}.
As a next step, we need to project the obtained result (decomposed matrix of rank $2r$) onto the tangent plane.
The operation can also be done efficiently via~(\ref{eq:proj_fixed_rank}) and is summarized in Algorithm~\ref{alg:transport} called \texttt{ProjectLR}.

	\section{Locally optimal initialization (LOI)}%
	\label{sec:locally_optimal_initialization}
	
    Once the theoretical framework for the Riemannian optimizer is established, it is natural to consider an initialization scheme that accounts for the underlying Riemannian geometry.
    Given any $\Delta W \in \mathcal{M}_r$, we may write
	\begin{equation}\label{eq:intro_splitting}
		\mathcal{L}(W) = 
		\mathcal{L}(\underbrace{W - \Delta W}_{W'} + \Delta W) =
		\mathcal{L}\left(  W' + \Delta W\right).
	\end{equation}
    This raises the question: how should $\Delta W$ be chosen to ensure the fastest loss decrease along the manifold?
	The solution is to consider the following optimization task:
	\begin{equation}\label{eq:inital_opt_task}
		\Delta W^{(0)}_{*} \in  
		\Argmax_{\Delta W \in \mathcal{M}_r}
		\left\|P_{\mathcal{T}_{\Delta W} \mathcal{M}_r}\,\nabla_W
		\mathcal{L}(W)\right\|^2_F.
	\end{equation}
	Since $P_{\mathcal{T}_{\Delta W} \mathcal{M}_r}$ is an orthogonal projection matrix to the tangent plane, the task (\ref{eq:inital_opt_task}) essentially  seeks for the point on 
	the fixed-rank manifold, 
	whose tangent space has most alignment with the Euclidean gradient. In other words, this means that the direction of the steepest 
	local function decrease alongside the manifold is aligned 
	with the full model tuning direction.
	The solution to this task is presented in  Theorem~\ref{theorem:optimal_initialization}.
\begin{theorem}\label{theorem:optimal_initialization}
        Let the SVD of $\nabla_W \mathcal{L}(W)$ be:
		\begin{equation*}
			\nabla_W \mathcal{L}(W) = 
			\begin{bmatrix} 
				U_{1, r} & U_{r, 2r} & U_\perp
			\end{bmatrix} 
			\begin{bmatrix} 
				\Sigma_{1,r} & 0 & 0 \\
				0 & \Sigma_{r,2r} & 0 \\
				0 & 0 & \Sigma_{\perp} 
			\end{bmatrix} 
			\begin{bmatrix} 
				V_{1, r} & V_{r, 2r} & V_\perp 
			\end{bmatrix}^\top,
		\end{equation*}
        and let also $\sigma_{2r}\not = \sigma_{2r+1}$.
		Then any optimal solution $\Delta W^{(0)}_{*}$ to the problem~(\ref{eq:inital_opt_task}) has the form:
    \begin{equation}\label{eq:special_init_form}
            \begin{split}
                \Delta W^{(0)}_{*} \in \Bigg\{
                &
                    \begin{bmatrix} 
                        U_{1, r}, U_{r, 2r} \Sigma _{r, 2r}
                    \end{bmatrix} 
                    \begin{bmatrix}
                        S_{11} \\ S_{21}        
                    \end{bmatrix}
                    \begin{bmatrix}
                        C_{21} & C_{22}        
                    \end{bmatrix}
                    \begin{bmatrix}
                        \Sigma_{1, r} V_{1, r}^\top \\
                        V_{r, 2r}
                    \end{bmatrix} \Bigg|\, 
                    \\
                    &
                    S =
                    \begin{bmatrix}
                        S_{11} & S_{12} \\
                        S_{21} & S_{22}
                    \end{bmatrix}
                    \in \mathrm{GL}_{2r}(\mathbb{R}), \  S^{-1} = \begin{bmatrix}
                        C_{11} & C_{12} \\
                        C_{21} & C_{22}
                    \end{bmatrix}
                \Bigg\}.
            \end{split}   
            \end{equation}
            	\end{theorem}
	\begin{proof}
	See Appendix \ref{sec:inital_point_search}.
	\end{proof}

    In our experiments, we use $S = \begin{bmatrix}
                        \alpha I_r & 0 \\
                        0 & I_r
                    \end{bmatrix} $, obtaining:
        $\Delta W^{(0)}_{*} = \alpha U_{1,r}V^\top_{r,2r} 
			\in \mathcal{M}_r,  \alpha \in \mathbb{R}\setminus \{0\}$.

    Interestingly, Theorem~\ref{theorem:optimal_initialization} relates to the findings of \citet{wang2024lora}, 
    although their analysis neither adopts a Riemannian framework nor addresses parametrization-free optimization. 
    Our optimizer further differs from \citet{zhang2024riemannian} and \citet[LORO]{moparameter} in that it avoids inversion of the Gram matrix. 
    As a result, the method remains stable as $\|\Delta W^{(0)}_\ast\|\to 0$. 
    Empirical results indicate that initializing $\Delta W^{(0)}_\ast$ with a small norm leads to improved performance. 
    The procedure for selecting the scaling parameter $\alpha$ is described in Appendix~\ref{sec:hyperparameter}.

\section{Single backward-pass gradient trick}\label{sec:one-step gradient trick}
    This section introduces an efficient method for computing gradient-times–matrix in a matrix-free way, at a computational cost equivalent to a \emph{single backward pass}. We then apply this technique in both the LOI initialization procedure in Algorithm~\ref{alg:bpr_svd} and the Riemannion optimizer in Algorithm~\ref{alg:fr_riemannion}.
    
    The calculation of the full fine-tuning loss gradient 
	$\nabla_{W} \mathcal{L}\left( W \right)$ with pretrained parameters $W \in \mathbb{R}^{m \times n}$ is 
	computationally expensive. 
    At the same time, for our framework we only need to compute the products 
	$(\nabla_{W} \mathcal{L}\left( W\right)^\top M)$ or $(\nabla_{W} \mathcal{L}\left( W \right) N)$ for some $M \in \mathbb{R}^{m \times r}, N \in \mathbb{R}^{n \times r}$.
    Using the trick from~\citep{novikov2022automatic}, we may calculate both quantities simultaneously using a \emph{single forward-backward pass} with a doubled rank representation. 

    First, we initialize differentiable parameters $Z_1 = 0 \in \mathbb{R}^{m \times r}$ and $Z_2 = 0 \in \mathbb{R}^{n \times r}$ and perform a simple forward pass $L := \mathcal{L}\left( W + Z_1N^\top +  M Z_2^\top\right)$.
	This step does not violate the pipelines of LoRA framework since it is equivalent to a standard LoRA forward pass with special adapter:
    \begin{equation*}
		Z_1N^\top + 
			M Z_2^\top = \begin{bmatrix}
			    Z_1 & M
			\end{bmatrix}\begin{bmatrix}
			    N & Z_2
			\end{bmatrix}^\top = \tilde{A}\tilde{B}^\top, \quad \tilde{A} \in \mathbb{R}^{m \times 2r}, \tilde{B} \in \mathbb{R}^{n \times 2r}.
	\end{equation*}
    Then we invoke an autodiff algorithm for value $L$, which ensures us both:
	\begin{equation}\label{eq:general_mm_grad_bp_trick}
		\begin{split}
			&\nabla_{Z_1} L = \nabla_{Z_1}
			\mathcal{L}\left( W + Z_1N^\top + 
			M Z_2^\top\right)|_{Z_1=0, Z_2=0} 
			= \nabla_W \mathcal{L}(W)N, \\
			&\nabla_{Z_2} L =\nabla_{Z_2} 
			\mathcal{L}\left( W + Z_1N^\top + 
			MZ_2^\top\right)|_{Z_1=0, Z_2=0} 
			= \nabla_W \mathcal{L}(W)^\top M.
		\end{split}
	\end{equation}
    Note, that if $\Delta W = A_L B^\top = AB_R^\top$ from $W + \Delta W$, then one may take $N = B_R, M = A_L, Z_1 = 0, Z_ 2 = B, Y = W  + \Delta W$ and the exact same operations work.

    Notably, (\ref{eq:general_mm_grad_bp_trick}) is crucial for computation of the Riemannian gradient in (\ref{eq:proj_fixed_rank}). Therefore the proposed approach is a key building block that allows us to avoid forming the full fine-tune gradient loss and effectively compute matrix-matrix multiplications. This idea is the core for both: LOI initialization (Algorithm~\ref{alg:bpr_svd}) and Riemannion optimizer (Algorithm~\ref{alg:fr_riemannion}).
    
    \paragraph{Randomized SVD for efficient initialization} 
    The computation of the $2r$-truncated SVD of the full loss gradient, as required by Theorem~\ref{theorem:optimal_initialization}, has asymptotic complexity $\mathcal{O}(\min\{m, n\}mn)$, which may be infeasible for large-scale models. In addition, explicitly forming this gradient is itself computationally expensive, and thus we need to avoid it in practice. To overcome this problem, 
	we propose to use a randomized SVD with power iterations (see \citep{halko2009finding}), which we also enhance with our \emph{one-step gradient trick} as is 
	described in the Algorithm \ref{alg:bpr_svd}.
    In a nutshell, we need to compute $\left(\nabla_W \mathcal{L}(W) \nabla_W \mathcal{L}(W)^\top\right)^q Y$, where $Y  = \nabla_W \mathcal{L}(W) \Omega$ and $\Omega$ is sampled from 
	standard normal distribution.
    This iteration can be done in a robust manner using QR decompositions.
    The steps $2, 4, 5, 6$ correspond to a trick from~(\ref{eq:general_mm_grad_bp_trick}).  
    
	Overall, the LOI search procedure has asymptotic complexity $\mathcal{O}((m+n)r^2)$, plus $2(q+1)$ additional backward passes. Moreover, since LOI search is executed only once before fine-tuning (which typically involves thousands of backward passes), its runtime overhead is negligible, taking  merely $0.25\%$ of the total fine-tuning wall-clock time in our experiments.
    Section~\ref{sub:commonsense_reasoning_finetuning}, about $2$ thousand optimization steps were carried out and Algorithm~\ref{alg:bpr_svd} accounted for merely $0.25\%$ of the total fine-tuning wall-clock time.
	
	\begin{algorithm}[h!tp]
		\caption{\texttt{BackPropRSVD}}
		\label{alg:bpr_svd}
		\textbf{Require:}
			Weights $W \in \mathbb{R}^{m \times n}$,
			  rank $r \in \mathbb{N}$,
			oversampling parameter $p$,
			power-step parameter $q$.\\
		\textbf{Ensure:} 
            Randomized $r$-truncated SVD $(U_{r}, \Sigma_{r}, V_{r})$ of $\nabla_W \mathcal{L}(W)$.

		\begin{algorithmic}[1]
			\STATE Choose $k = r + p$, Sample $\Omega \in \mathbb{R}^{n \times k} 
			\sim \mathcal{N}(0, 1)$. 
			\hfill\COMMENT{$\mathcal{O}\left(nr\right)$}

			\STATE  $Y := 
			\mathtt{qr}(\nabla_A \,\mathcal{L} (W + A\Omega^\top)|_{A = 0}).\mathtt{Q}.$
            \hfill\COMMENT{$1$ backward pass $+\mathcal{O}\left(mr^2\right)$}

			\FOR {$i := 1, \dots, q$}

				\STATE  $Y := 
				\mathtt{qr}([\nabla_B \,\mathcal{L} (W + YB^\top)|_{B = 0}]^\top).\mathtt{Q}.$ 
                \hfill\COMMENT{$1$ backward pass $+\mathcal{O}\left(nr^2\right)$}

				\STATE  $Y := 
				\mathtt{qr}(\nabla_A \,\mathcal{L} (W + AY^\top)|_{A = 0}).\mathtt{Q}.$ 
                \hfill\COMMENT{$1$ backward pass$+\mathcal{O}\left(mr^2\right)$}

			\ENDFOR

			\STATE  $Y := 
			[\nabla_B \,\mathcal{L} (W + YB^\top)|_{B = 0}]^\top.$ 
            \hfill\COMMENT{$1$ backward pass}
			\STATE  $U, \Sigma, V^\top := 
			\mathtt{truncSVD}\left(Y, r\right).$
			\hfill\COMMENT{$\mathcal{O}\left(nr^2\right)$}
			\STATE $YU, \Sigma, V^\top$ 
			\hfill\COMMENT{$\mathcal{O}\left(mr^2\right)$}
		\end{algorithmic}
	\end{algorithm}

    \paragraph{Efficient Riemannion implementation}
    The  Riemannion optimizer is presented in Algorithm~\ref{alg:fr_riemannion}. Similar to vanilla Muon, it relies on a chosen \texttt{Ortho} procedure. Since the LoRA approach represents the fine-tuning shift $\Delta W$ with low-rank matrices, we employ an SVD-based \texttt{Ortho} procedure. For computational efficiency, this procedure is adapted to the fixed-rank manifold, as described in Section~\ref{sec:riemannion} and implemented in the \texttt{OrthoLR}  (Algorithm~\ref{alg:fr_procrustes}). 
	In detail, in step 1 we  calculate the Riemannian gradient components via  a \emph{single backward call}~(\ref{eq:general_mm_grad_bp_trick}). Then, in step $2$ the algorithm transports the Heavy-Ball tangent
	direction to the current point via Algorithm \ref{alg:transport} that provides a simple but effective implementation of~(\ref{eq:proj_fixed_rank}). In line $4$, we compute
	the final optimization direction on the tangent space of the given point 
	with a Heavy-Ball momentum coefficient $\beta$. In line $5$ algorithm performs the retraction step.

	The algorithm finalizes with saving the obtained minimization direction 
	for momentum in line $6$ and calculating of a new point representation in step $7$.

    \begin{algorithm}[h!tp]
		\caption{\texttt{One step of Riemannion}}
		\label{alg:fr_riemannion}
		\textbf{Require:}
			Weight matrix $W' \in \mathbb{R}^{m \times n}$, rank $r \in \mathbb{N}$,
            initial point $A_L, B_R$,
            Heavy-ball momentum $A_{\mathrm{HB}}, B_{\mathrm{HB}}$,
			step size $\eta$,
			momentum coefficient $\beta$,
            weight decay coefficient $\gamma$.\\
		\textbf{Ensure:}
			Tuning parameters $\Delta W^* \in \mathcal{M}_r$.
		\begin{algorithmic}[1]

				\STATE $\dot{A}, \dot{B} := \nabla_{Z_1} 
							\mathcal{L}\left( W' + Z_1B_R^\top + 
							A_L Z_2^\top\right)|_{Z_1=0, Z_2=B},$ \\
                            \hspace{1.27cm}$ \nabla_{Z_2} 
							\mathcal{L}\left(W'+ Z_1B_R^\top + 
							A_L Z_2^\top\right)|_{Z_1=0, Z_2=B}.$
                \hfill\COMMENT{1 backward pass}

				\STATE $\dot{A}_{\mathrm{prev}}, \dot{B}_{\mathrm{prev}} := 
				\mathtt{ProjectLR}\left((A_{\mathrm{M}}, B_{\mathrm{M}}), A_L, B_R\right)$. 
				\hfill\COMMENT{$\mathcal{O}\left((m + n)r^2\right)$}

				\STATE $\dot{A}, \dot{B} := \beta \dot{A}_{\mathrm{prev}} + 
				(I - A_L A_L^\top) \dot{A}, \;
                \beta \dot{B}_{\mathrm{prev}} + 
				\dot{B}$ 
				\hfill\COMMENT{$\mathcal{O}\left((n+m)r^2\right)$}

                \STATE $\dot A, \dot B := \mathtt{ProjectLR}(\mathtt{OrthoLR} (A_L, B_R, \dot A, \dot B, r, ), A_L, B_R)$
                \hfill \COMMENT{$\mathcal{O}\left((m + n)r^2 +r^3\right)$}

				\STATE 
				$
					U, \Sigma, V^\top := 
					\mathtt{RetractionLR}( [-\eta\dot{A},  A_L],
					[  B_R, -\eta(\dot{B} + \gamma B_R) ] ) 
				$ \hfill \COMMENT{$\mathcal{O}\left((m + n)r^2 +r^3\right)$}

				\STATE $A_{\mathrm{HB}}, B_{\mathrm{HB}} := [\dot{A}, A_L], \; [B_R, \dot{B}]$ 

				\STATE $A_L, B := U, \Sigma V^\top$ 
				\hfill\COMMENT{$\mathcal{O}\left(nr^2\right)$}
		\end{algorithmic}
	\end{algorithm}

    Overall, one iteration of the Riemannion loop has asymptotic complexity $\mathcal{O}((m+n)r^2 + r^3)$ and additionally the same number of backward passes as vanilla LoRA. For comparison, the Euclidean Muon optimizer for LoRA (Section~\ref{sec:muon}) exhibits the same asymptotic complexity. The complete framework and its time performance are summarized in  Algorithm~\ref{alg:full_fr_riemannion} and Appendix~\ref{sec:overhead}.

\section{Experiments}%
\label{sec:experiments}
In this section, we also employ the term LOI (Locally Optimal Initialization), referring to the proposed initialization scheme described in Section~\ref{sec:locally_optimal_initialization}. 
This is done to distinguish between using  RiemannLoRA as an optimizer with zero initialization (as in basic LoRA) and RiemannLoRA-LOI, which is the proposed combination of initialization and optimization. We conduct a series of LLM fine-tuning experiments for different tasks. All of the experiments were computed on  NVIDIA V100-32Gb GPU and A100-80Gb GPU. We ran all the experiments within $\sim 2000$ GPU hours.
\subsection{Commonsense reasoning fine-tuning}%
\label{sub:commonsense_reasoning_finetuning}

The results were obtained for the benchmark (\citet[BoolQ]{clark2019boolq},
\citet[PIQA]{bisk2020piqa}, \citet[SIQA]{sap2019socialiqa}, 
\citet[hellaswag]{zellers2019hellaswag}, 
\citet[winogrande]{sakaguchi2021winogrande},\citet[ARC]{clark2018arc},
\citet[OBQA]{OpenBookQA2018})
common reasoning. The structure of the dataset is described in  
Appendix~\ref{sec:dataset}.
In the following experiments
we conduct a fine-tuning procedures for multilayer perceptron (MLP) and attention layers of
Llama 3 8b model (\citet{grattafiori2024llama}).

The commonsense reasoning tasks comprise of $8$ sub-tasks, each of them 
contains a predefined training and a testing set. 
We follow the setting of \citet{hu2023llm}
and amalgamate the training datasets from all $8$ tasks to create the final 
training dataset and conduct evaluations on the individual testing dataset
for each task.

The hyperparameter tuning protocol and the selected hyperparameters are provided in \ref{sec:hyperparameter}. Table \ref{table:cr_adam_8b_comp} contains the accuracy 
of the trained model's responses on the test dataset. 
Within the LoRA framework, the proposed Riemannian fine-tuning method delivers clear performance gains. The Riemannion optimizer consistently outperforms LoRA, DoRA, and achieves superior metric results compared to the standard Muon optimizer applied to LoRA factors (see Section~\ref{sec:muon}). Furthermore, relative to other Riemannian-geometry–aware approaches such as RPrecAdamW \citep{zhang2024riemannian}, our method also demonstrates better results. Finally, the variance of outcomes for the proposed method is the smallest among all compared approaches.

\begin{table}[h!tp]
    \centering
    \caption{The average accuracy (in \%) among $8$ tasks of fine-tuned Llama 3-8b   
    using different approaches, tested on 
    Commonsense Reasoning benchmark. LoRA rank is set to 16. }
	\scriptsize
    \setlength{\tabcolsep}{3pt}
    \begin{tabular}{lrrrrrrrrr}
        \toprule
        Task & BoolQ & PIQA & SIQA & hella-& wino-& ARC-E& ARC-C& OBQA & All \\
        Initialization &  &  &  &  swag &  grande &  &  & &  \\
        \midrule
        Raw & $65.0$ & $76.6$ & $73.0$ & $66.1$ & $61.3$ & $92.5$ & $82.3$ & $79.6$ & $74.5$ \\
        Adam & $74.8_{\pm 1.9}$ & $89.8_{\pm 0.9}$ & $82.6_{\pm 0.6}$ & $96.2_{\pm 0.3}$ & $87.9_{\pm 1.2}$ & $92.4_{\pm 0.7}$ & $84.9_{\pm 0.7}$ & $88.5_{\pm 0.4}$ & $87.1_{\pm 0.6}$ \\
        DoRA & $74.8_{\pm 0.8}$ & $89.4_{\pm 0.5}$ & $82.4_{\pm 0.7}$ & $95.9_{\pm 0.1}$ & $87.8_{\pm 0.4}$ & $90.7_{\pm 1.2}$ & $83.8_{\pm 0.7}$ & $87.8_{\pm 0.6}$ & $86.6_{\pm 0.3}$ \\
        Muon & $72.9_{\pm 0.0}$ & $86.4_{\pm 0.5}$ & $80.8_{\pm 0.2}$ & $94.1_{\pm 0.2}$ & $84.4_{\pm 0.0}$ & $84.2_{\pm 0.9}$ & $77.3_{\pm 2.5}$ & $83.9_{\pm 1.1}$ & $83.0_{\pm 0.6}$ \\
        DoneRITE & $72.2_{\pm 0.3}$ & $88.6_{\pm 0.1}$ & $82.0_{\pm 0.6}$ & $95.1_{\pm 0.1}$ & $85.6_{\pm 0.2}$ & $87.7_{\pm 1.5}$ & $79.3_{\pm 2.6}$ & $85.7_{\pm 0.5}$ & $84.5_{\pm 0.5}$ \\
        RPrecAdamW & $\mathbf{75.8}_{\pm 0.4}$ & $89.5_{\pm 0.4}$ & $82.4_{\pm 0.2}$ & $96.1_{\pm 0.2}$ & $87.7_{\pm 0.9}$ & $90.6_{\pm 1.6}$ & $84.1_{\pm 1.1}$ & $87.7_{\pm 0.5}$ & $86.8_{\pm 0.4}$ \\

        Riemannion & $75.7_{\pm 0.7}$ & $\mathbf{91.2_{\pm 0.2}}$ & $\mathbf{83.5_{\pm 0.6}}$ & $\mathbf{96.7}_{\pm 0.0}$ & $\mathbf{88.6}_{\pm 0.4}$ & $\mathbf{93.6}_{\pm 0.3}$ & $\mathbf{86.4_{\pm 0.4}}$ & $\mathbf{89.3}_{\pm 0.8}$ & $\mathbf{88.1}_{\pm 0.2}$ \\
        \bottomrule
    \end{tabular}
    \label{table:cr_adam_8b_comp}
\end{table}

\subsection{Subject-driven generation}%
\label{sub:subject_driven_generation}

Subject-driven generation \citep{DB, TI} is a task in which the user provides several reference photos of an object, called a concept, and uses a diffusion model to generate this concept with certain conditions (e.g., a textual prompt).  One way to solve this task is to fine-tune a pre-trained diffusion model using this small set of reference images. However, this technique leads to a degradation in understanding of the conditions and to a fast overfitting of the concept. Furthermore, it requires a high computational cost due to the large number of trainable parameters. This is why previous works, such as \citep{qiu2023controlling, liu2024parameterefficient, hu2022lora, r1e, svdiff, gorbunov2024groupshuffleefficientstructured}, propose training only a lightweight parameterization for the base model. In this section, we demonstrate the performance of our parameterization in this task.

In our experiments, we used Stable Diffusion 2 \citep{Rombach_2022_CVPR} as the base model. We choose LoRA \citet{hu2022lora} as a baseline and train both models with ranks of $4$, $8$ and $16$. 
\begin{wrapfigure}{r}{0.4\textwidth}
    \centering
    \includegraphics[width=\linewidth]{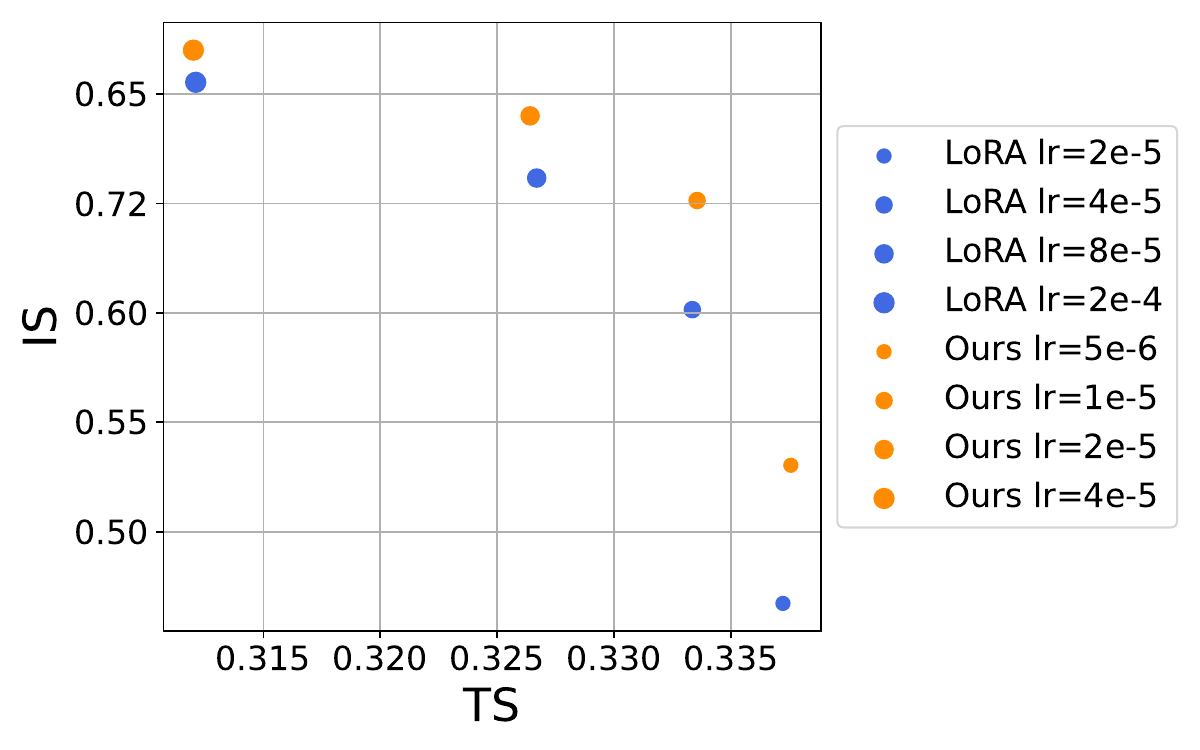}
    \caption{Comparison of text and image similarities for LoRA and our method with rank 4 at different learning rates on 400 step.}
    \label{fig:sdg_metr}
\end{wrapfigure}
We predict the parameterization of the \texttt{q}, \texttt{k}, \texttt{v}, and \texttt{out.0} matrices in all attention layers. The Dreambooth dataset  \citep{DB} was used in all our experiments. LoRA was trained using the Adam optimizer.
In Figure  \ref{fig:sdg_visual}, we present a visual comparison of LoRA with different ranks. As can be seen, even for complex concepts such as 'robot toy', our method requires only 600 steps to learn the concept while preserving appropriate text similarity. We found that for subject-driven generation tasks, the lower the rank, the faster our method converges. To evaluate this, we also calculated metrics for different learning rates on a subset of the Dreambooth dataset. We use CLIP to measure text similarity and DINO to measure image similarity. Figure \ref{fig:sdg_metr} shows that, even with different learning rates, our method achieves more accurate results in concept preservation. Further details and a visual comparison can be found in the Appendix \ref{sec:sdg}.

\begin{figure}[t]
    \centering
    \includegraphics[width=0.9\linewidth]{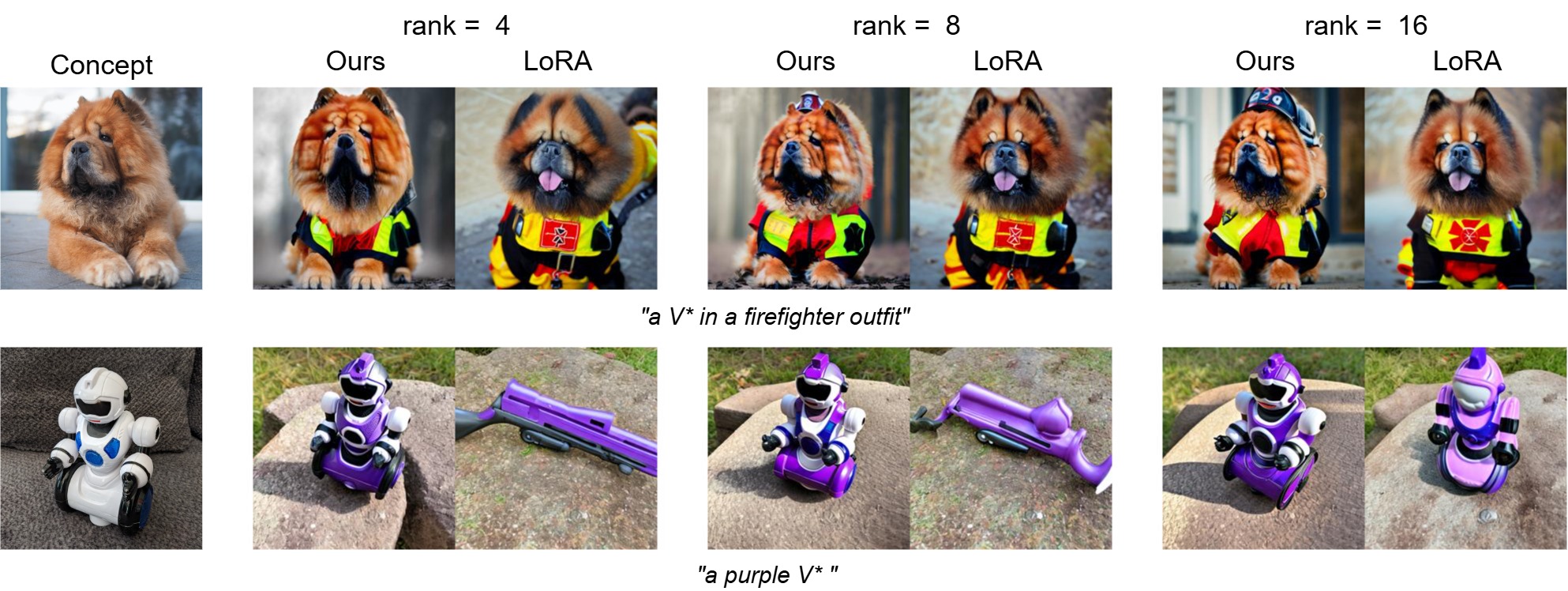}
    \caption{Visual results for Subject-driven generation on 600 training step.}
    \label{fig:sdg_visual}
\end{figure}

\section{Conclusion}%
\label{sec:conclusion}
	In this work, we propose a novel fully Riemannian framework that integrates a new muon-based optimization method, locally-optimal initialization, and an efficient implementation. This integrated approach yields a reliable reparametrization-invariant method that outperforms competing approaches on fine-tuning large language models (LLMs) and exhibits additional favorable properties for low-rank approximations in diffusion models.
Given the promising empirical results, a natural direction for future research is to investigate the theoretical properties of the proposed method.

\section{Reproducibility Statement}
The hyperparameter selection procedure is described in Appendix \ref{sec:hyperparameter}.
The datasets used in the experiments and the corresponding preprocessing steps are detailed in Appendix \ref{sec:dataset}.
The proof of the stated assumptions is provided in Appendix \ref{sec:inital_point_search}. The hyperparameters for Subject-driven generation task are provided in the Appendix \ref{sec:sdg}.

\bibliography{iclr2026_conference}

\begin{thebibliography}{53}
\providecommand{\natexlab}[1]{#1}
\providecommand{\url}[1]{\texttt{#1}}
\expandafter\ifx\csname urlstyle\endcsname\relax
  \providecommand{\doi}[1]{doi: #1}\else
  \providecommand{\doi}{doi: \begingroup \urlstyle{rm}\Url}\fi

\bibitem[Absil \& Oseledets(2015)Absil and Oseledets]{absil2015low}
P-A Absil and Ivan~V Oseledets.
\newblock Low-rank retractions: a survey and new results.
\newblock \emph{Computational Optimization and Applications}, 62\penalty0 (1):\penalty0 5--29, 2015.

\bibitem[Absil et~al.(2009)Absil, Mahony, and Sepulchre]{absil2009optimization}
P-A Absil, Robert Mahony, and Rodolphe Sepulchre.
\newblock Optimization algorithms on matrix manifolds.
\newblock In \emph{Optimization Algorithms on Matrix Manifolds}. Princeton University Press, 2009.

\bibitem[Baker(2008)]{baker2008riemannian}
Christopher~G Baker.
\newblock \emph{Riemannian manifold trust-region methods with applications to eigenproblems}.
\newblock The Florida State University, 2008.

\bibitem[Bernstein(2025)]{bernstein2025deriving}
Jeremy Bernstein.
\newblock Deriving muon, 2025.
\newblock URL \url{https://jeremybernste.in/writing/deriving-muon}.

\bibitem[Bisk et~al.(2020)Bisk, Zellers, Gao, Choi, et~al.]{bisk2020piqa}
Yonatan Bisk, Rowan Zellers, Jianfeng Gao, Yejin Choi, et~al.
\newblock P{IQA}: {R}easoning about {P}hysical {C}ommonsense in {N}atural {L}anguage.
\newblock In \emph{Proceedings of the AAAI conference on artificial intelligence}, number~05 in 34, pp.\  7432--7439, 2020.

\bibitem[Boumal(2023)]{boumal2023introduction}
Nicolas Boumal.
\newblock \emph{An introduction to optimization on smooth manifolds}.
\newblock Cambridge University Press, 2023.

\bibitem[Boyd \& Dattorro(2003)Boyd and Dattorro]{boyd2003}
Stephen Boyd and Jon Dattorro.
\newblock Alternating projections.
\newblock 2003.
\newblock URL \url{https://web.stanford.edu/class/ee392o/alt_proj.pdf}.

\bibitem[Brown et~al.(2020)Brown, Mann, Ryder, Subbiah, Kaplan, Dhariwal, Neelakantan, Shyam, Sastry, Askell, et~al.]{brown2020language}
Tom Brown, Benjamin Mann, Nick Ryder, Melanie Subbiah, Jared~D Kaplan, Prafulla Dhariwal, Arvind Neelakantan, Pranav Shyam, Girish Sastry, Amanda Askell, et~al.
\newblock Language models are few-shot learners.
\newblock \emph{Advances in neural information processing systems}, 33:\penalty0 1877--1901, 2020.

\bibitem[Cesista(2025)]{cesista2025sdnr}
Franz~Louis Cesista.
\newblock {M}uon and a selective survey on {S}teepest {D}escent in {R}iemannian and non-{R}iemannian {M}anifolds, April 2025.
\newblock URL \url{http://leloykun.github.io/ponder/steepest-descent-non-riemannian/}.

\bibitem[Cheney \& Goldstein(1959)Cheney and Goldstein]{cheney1959proximity}
Ward Cheney and Allen~A Goldstein.
\newblock Proximity maps for convex sets.
\newblock \emph{Proceedings of the American Mathematical Society}, 10\penalty0 (3):\penalty0 448--450, 1959.

\bibitem[Clark et~al.(2019)Clark, Lee, Chang, Kwiatkowski, Collins, and Toutanova]{clark2019boolq}
Christopher Clark, Kenton Lee, Ming{-}Wei Chang, Tom Kwiatkowski, Michael Collins, and Kristina Toutanova.
\newblock Bool{Q}: Exploring the surprising difficulty of natural yes/no questions.
\newblock In Jill Burstein, Christy Doran, and Thamar Solorio (eds.), \emph{Proceedings of the 2019 Conference of the North American Chapter of the Association for Computational Linguistics: Human Language Technologies, {NAACL-HLT} 2019, Minneapolis, MN, USA, June 2-7, 2019, Volume 1 (Long and Short Papers)}, pp.\  2924--2936. Association for Computational Linguistics, 2019.
\newblock \doi{10.18653/V1/N19-1300}.
\newblock URL \url{https://doi.org/10.18653/v1/n19-1300}.

\bibitem[Clark et~al.(2018)Clark, Cowhey, Etzioni, Khot, Sabharwal, Schoenick, and Tafjord]{clark2018arc}
Peter Clark, Isaac Cowhey, Oren Etzioni, Tushar Khot, Ashish Sabharwal, Carissa Schoenick, and Oyvind Tafjord.
\newblock Think you have solved question answering? try {ARC}, the {AI2} reasoning challenge.
\newblock \emph{CoRR}, abs/1803.05457, 2018.
\newblock URL \url{http://arxiv.org/abs/1803.05457}.

\bibitem[Dubey et~al.(2024)Dubey, Jauhri, Pandey, Kadian, Al{-}Dahle, Letman, Mathur, Schelten, Yang, Fan, Goyal, Hartshorn, Yang, Mitra, Sravankumar, Korenev, Hinsvark, Rao, Zhang, Rodriguez, Gregerson, Spataru, Rozi{\`{e}}re, Biron, Tang, Chern, Caucheteux, Nayak, Bi, Marra, McConnell, Keller, Touret, Wu, Wong, Ferrer, Nikolaidis, Allonsius, Song, Pintz, Livshits, Esiobu, Choudhary, Mahajan, Garcia{-}Olano, Perino, Hupkes, Lakomkin, AlBadawy, Lobanova, Dinan, Smith, Radenovic, Zhang, Synnaeve, Lee, Anderson, Nail, Mialon, Pang, Cucurell, Nguyen, Korevaar, Xu, Touvron, Zarov, Ibarra, Kloumann, Misra, Evtimov, Copet, Lee, Geffert, Vranes, Park, Mahadeokar, Shah, van~der Linde, Billock, Hong, Lee, Fu, Chi, Huang, Liu, Wang, Yu, Bitton, Spisak, Park, Rocca, Johnstun, Saxe, Jia, Alwala, Upasani, Plawiak, Li, Heafield, Stone, and et~al.]{grattafiori2024llama}
Abhimanyu Dubey, Abhinav Jauhri, Abhinav Pandey, Abhishek Kadian, Ahmad Al{-}Dahle, Aiesha Letman, Akhil Mathur, Alan Schelten, Amy Yang, Angela Fan, Anirudh Goyal, Anthony Hartshorn, Aobo Yang, Archi Mitra, Archie Sravankumar, Artem Korenev, Arthur Hinsvark, Arun Rao, Aston Zhang, Aur{\'{e}}lien Rodriguez, Austen Gregerson, Ava Spataru, Baptiste Rozi{\`{e}}re, Bethany Biron, Binh Tang, Bobbie Chern, Charlotte Caucheteux, Chaya Nayak, Chloe Bi, Chris Marra, Chris McConnell, Christian Keller, Christophe Touret, Chunyang Wu, Corinne Wong, Cristian~Canton Ferrer, Cyrus Nikolaidis, Damien Allonsius, Daniel Song, Danielle Pintz, Danny Livshits, David Esiobu, Dhruv Choudhary, Dhruv Mahajan, Diego Garcia{-}Olano, Diego Perino, Dieuwke Hupkes, Egor Lakomkin, Ehab AlBadawy, Elina Lobanova, Emily Dinan, Eric~Michael Smith, Filip Radenovic, Frank Zhang, Gabriel Synnaeve, Gabrielle Lee, Georgia~Lewis Anderson, Graeme Nail, Gr{\'{e}}goire Mialon, Guan Pang, Guillem Cucurell, Hailey Nguyen, Hannah Korevaar, Hu~Xu, Hugo
  Touvron, Iliyan Zarov, Imanol~Arrieta Ibarra, Isabel~M. Kloumann, Ishan Misra, Ivan Evtimov, Jade Copet, Jaewon Lee, Jan Geffert, Jana Vranes, Jason Park, Jay Mahadeokar, Jeet Shah, Jelmer van~der Linde, Jennifer Billock, Jenny Hong, Jenya Lee, Jeremy Fu, Jianfeng Chi, Jianyu Huang, Jiawen Liu, Jie Wang, Jiecao Yu, Joanna Bitton, Joe Spisak, Jongsoo Park, Joseph Rocca, Joshua Johnstun, Joshua Saxe, Junteng Jia, Kalyan~Vasuden Alwala, Kartikeya Upasani, Kate Plawiak, Ke~Li, Kenneth Heafield, Kevin Stone, and et~al.
\newblock The {L}lama 3 herd of models.
\newblock \emph{CoRR}, abs/2407.21783, 2024.
\newblock \doi{10.48550/ARXIV.2407.21783}.
\newblock URL \url{https://doi.org/10.48550/arXiv.2407.21783}.

\bibitem[Duchi et~al.(2011)Duchi, Hazan, and Singer]{duchi2011adaptive}
John Duchi, Elad Hazan, and Yoram Singer.
\newblock Adaptive subgradient methods for online learning and stochastic optimization.
\newblock \emph{Journal of machine learning research}, 12\penalty0 (7), 2011.

\bibitem[Eckart \& Young(1936)Eckart and Young]{eckart1936approximation}
Carl Eckart and Gale Young.
\newblock The approximation of one matrix by another of lower rank.
\newblock \emph{Psychometrika}, 1\penalty0 (3):\penalty0 211--218, 1936.

\bibitem[Gal et~al.(2023)Gal, Alaluf, Atzmon, Patashnik, Bermano, Chechik, and Cohen{-}Or]{TI}
Rinon Gal, Yuval Alaluf, Yuval Atzmon, Or~Patashnik, Amit~Haim Bermano, Gal Chechik, and Daniel Cohen{-}Or.
\newblock An image is worth one word: Personalizing text-to-image generation using textual inversion.
\newblock In \emph{The Eleventh International Conference on Learning Representations, {ICLR} 2023, Kigali, Rwanda, May 1-5, 2023}. OpenReview.net, 2023.
\newblock URL \url{https://openreview.net/forum?id=NAQvF08TcyG}.

\bibitem[Gorbunov et~al.(2024)Gorbunov, Yudin, Soboleva, Alanov, Naumov, and Rakhuba]{gorbunov2024groupshuffleefficientstructured}
Mikhail Gorbunov, Nikolay Yudin, Vera Soboleva, Aibek Alanov, Alexey Naumov, and Maxim Rakhuba.
\newblock Group and shuffle: Efficient structured orthogonal parametrization.
\newblock In Amir Globersons, Lester Mackey, Danielle Belgrave, Angela Fan, Ulrich Paquet, Jakub~M. Tomczak, and Cheng Zhang (eds.), \emph{Advances in Neural Information Processing Systems 38: Annual Conference on Neural Information Processing Systems 2024, NeurIPS 2024, Vancouver, BC, Canada, December 10 - 15, 2024}, 2024.

\bibitem[Halko et~al.(2011)Halko, Martinsson, and Tropp]{halko2009finding}
N.~Halko, P.~G. Martinsson, and J.~A. Tropp.
\newblock Finding structure with randomness: Probabilistic algorithms for constructing approximate matrix decompositions.
\newblock \emph{SIAM Review}, 53\penalty0 (2):\penalty0 217--288, 2011.
\newblock \doi{10.1137/090771806}.
\newblock URL \url{https://doi.org/10.1137/090771806}.

\bibitem[Han et~al.(2023)Han, Li, Zhang, Milanfar, Metaxas, and Yang]{svdiff}
Ligong Han, Yinxiao Li, Han Zhang, Peyman Milanfar, Dimitris Metaxas, and Feng Yang.
\newblock Svdiff: Compact parameter space for diffusion fine-tuning.
\newblock In \emph{Proceedings of the IEEE/CVF International Conference on Computer Vision}, pp.\  7323--7334, 2023.

\bibitem[Hayou et~al.(2024)Hayou, Ghosh, and Yu]{hayou2024lora+}
Soufiane Hayou, Nikhil Ghosh, and Bin Yu.
\newblock Lo{RA}+: {E}fficient {L}ow {R}ank {A}daptation of {L}arge {M}odels.
\newblock In \emph{Forty-first International Conference on Machine Learning, {ICML} 2024, Vienna, Austria, July 21-27, 2024}. OpenReview.net, 2024.
\newblock URL \url{https://openreview.net/forum?id=NEv8YqBROO}.

\bibitem[Hu et~al.(2022)Hu, Shen, Wallis, Allen-Zhu, Li, Wang, Wang, Chen, et~al.]{hu2022lora}
Edward~J Hu, Yelong Shen, Phillip Wallis, Zeyuan Allen-Zhu, Yuanzhi Li, Shean Wang, Lu~Wang, Weizhu Chen, et~al.
\newblock Lo{RA}: Low-{R}ank {A}daptation of {L}arge {L}anguage {M}odels.
\newblock \emph{ICLR}, 1\penalty0 (2):\penalty0 3, 2022.

\bibitem[Hu et~al.(2024)Hu, Cui, Lin, Wen, Li, et~al.]{hu2024retraction}
Jiang Hu, Jiaxi Cui, Lin Lin, Zaiwen Wen, Quanzheng Li, et~al.
\newblock {R}etraction-free optimization over the {S}tiefel manifold with application to the lo{RA} fine-tuning, 2024.

\bibitem[Hu et~al.(2023)Hu, Wang, Lan, Xu, Lim, Bing, Xu, Poria, and Lee]{hu2023llm}
Zhiqiang Hu, Lei Wang, Yihuai Lan, Wanyu Xu, Ee{-}Peng Lim, Lidong Bing, Xing Xu, Soujanya Poria, and Roy~Ka{-}Wei Lee.
\newblock {LLM}-adapters: An adapter family for parameter-efficient fine-tuning of large language models.
\newblock In Houda Bouamor, Juan Pino, and Kalika Bali (eds.), \emph{Proceedings of the 2023 Conference on Empirical Methods in Natural Language Processing, {EMNLP} 2023, Singapore, December 6-10, 2023}, pp.\  5254--5276. Association for Computational Linguistics, 2023.
\newblock \doi{10.18653/V1/2023.EMNLP-MAIN.319}.
\newblock URL \url{https://doi.org/10.18653/v1/2023.emnlp-main.319}.

\bibitem[Jordan et~al.(2024)Jordan, Jin, Boza, You, Cesista, Newhouse, and Bernstein]{jordan2024muon}
Keller Jordan, Yuchen Jin, Vlado Boza, Jiacheng You, Franz Cesista, Laker Newhouse, and Jeremy Bernstein.
\newblock Muon: An optimizer for hidden layers in neural networks, 2024.
\newblock URL \url{https://kellerjordan.github.io/posts/muon/}.

\bibitem[Kingma \& Ba(2014)Kingma and Ba]{kingma2014adam}
Diederik Kingma and Jimmy Ba.
\newblock Adam: A method for stochastic optimization.
\newblock \emph{International Conference on Learning Representations}, 12 2014.

\bibitem[Lee(2003)]{lee2003smooth}
John~M Lee.
\newblock \emph{Smooth manifolds}.
\newblock Springer, 2003.

\bibitem[Liu et~al.(2024)Liu, Qiu, Feng, Xiu, Xue, Yu, Feng, Liu, Heo, Peng, Wen, Black, Weller, and Sch{\"o}lkopf]{liu2024parameterefficient}
Weiyang Liu, Zeju Qiu, Yao Feng, Yuliang Xiu, Yuxuan Xue, Longhui Yu, Haiwen Feng, Zhen Liu, Juyeon Heo, Songyou Peng, Yandong Wen, Michael~J. Black, Adrian Weller, and Bernhard Sch{\"o}lkopf.
\newblock Parameter-efficient orthogonal finetuning via butterfly factorization.
\newblock In \emph{The Twelfth International Conference on Learning Representations}, 2024.
\newblock URL \url{https://openreview.net/forum?id=7NzgkEdGyr}.

\bibitem[Loshchilov \& Hutter(2019)Loshchilov and Hutter]{loshchilov2017decoupled}
Ilya Loshchilov and Frank Hutter.
\newblock {D}ecoupled {W}eight {D}ecay {R}egularization.
\newblock In \emph{7th International Conference on Learning Representations, {ICLR} 2019, New Orleans, LA, USA, May 6-9, 2019}. OpenReview.net, 2019.
\newblock URL \url{https://openreview.net/forum?id=Bkg6RiCqY7}.

\bibitem[Meng et~al.(2024)Meng, Wang, and Zhang]{meng2024pissa}
Fanxu Meng, Zhaohui Wang, and Muhan Zhang.
\newblock Pi{SSA}: {P}rincipal {S}ingular {V}alues and {S}ingular {V}ectors {A}daptation of {L}arge {L}anguage {M}odels.
\newblock \emph{Advances in Neural Information Processing Systems}, 37:\penalty0 121038--121072, 2024.

\bibitem[Mihaylov et~al.(2018)Mihaylov, Clark, Khot, and Sabharwal]{OpenBookQA2018}
Todor Mihaylov, Peter Clark, Tushar Khot, and Ashish Sabharwal.
\newblock Can a suit of armor conduct electricity? {A} new dataset for open book question answering.
\newblock In Ellen Riloff, David Chiang, Julia Hockenmaier, and Jun'ichi Tsujii (eds.), \emph{Proceedings of the 2018 Conference on Empirical Methods in Natural Language Processing, Brussels, Belgium, October 31 - November 4, 2018}, pp.\  2381--2391. Association for Computational Linguistics, 2018.
\newblock \doi{10.18653/V1/D18-1260}.
\newblock URL \url{https://doi.org/10.18653/v1/d18-1260}.

\bibitem[Mo et~al.(2025)Mo, Huang, and Pan]{moparameter}
Zhanfeng Mo, Long-Kai Huang, and Sinno~Jialin Pan.
\newblock Parameter and {M}emory {E}fficient {P}retraining via {L}ow-rank {R}iemannian {O}ptimization.
\newblock In \emph{The Thirteenth International Conference on Learning Representations}, 2025.

\bibitem[Novikov et~al.(2022)Novikov, Rakhuba, and Oseledets]{novikov2022automatic}
Alexander Novikov, Maxim Rakhuba, and Ivan Oseledets.
\newblock Automatic differentiation for {R}iemannian optimization on low-rank matrix and tensor-train manifolds.
\newblock \emph{SIAM Journal on Scientific Computing}, 44\penalty0 (2):\penalty0 A843--A869, 2022.

\bibitem[Parkina \& Rakhuba(2025)Parkina and Rakhuba]{parkina2025coalanumericallystableefficient}
Uliana Parkina and Maxim Rakhuba.
\newblock Coala: Numerically stable and efficient framework for context-aware low-rank approximation, 2025.
\newblock URL \url{https://arxiv.org/abs/2507.07580}.

\bibitem[Polyak(1964)]{polyak1964some}
Boris~T Polyak.
\newblock Some methods of speeding up the convergence of iteration methods.
\newblock \emph{Ussr computational mathematics and mathematical physics}, 4\penalty0 (5):\penalty0 1--17, 1964.

\bibitem[Qiu et~al.(2023)Qiu, Liu, Feng, Xue, Feng, Liu, Zhang, Weller, and Sch{\"o}lkopf]{qiu2023controlling}
Zeju Qiu, Weiyang Liu, Haiwen Feng, Yuxuan Xue, Yao Feng, Zhen Liu, Dan Zhang, Adrian Weller, and Bernhard Sch{\"o}lkopf.
\newblock Controlling text-to-image diffusion by orthogonal finetuning.
\newblock In \emph{Thirty-seventh Conference on Neural Information Processing Systems}, 2023.
\newblock URL \url{https://openreview.net/forum?id=K30wTdIIYc}.

\bibitem[Robbins \& Monro(1951)Robbins and Monro]{robbins1951stochastic}
Herbert Robbins and Sutton Monro.
\newblock A stochastic approximation method.
\newblock \emph{The annals of mathematical statistics}, pp.\  400--407, 1951.

\bibitem[Rombach et~al.(2022)Rombach, Blattmann, Lorenz, Esser, and Ommer]{Rombach_2022_CVPR}
Robin Rombach, Andreas Blattmann, Dominik Lorenz, Patrick Esser, and Bj\"orn Ommer.
\newblock High-resolution image synthesis with latent diffusion models.
\newblock In \emph{Proceedings of the IEEE/CVF Conference on Computer Vision and Pattern Recognition (CVPR)}, pp.\  10684--10695, June 2022.

\bibitem[Ruiz et~al.(2023)Ruiz, Li, Jampani, Pritch, Rubinstein, and Aberman]{DB}
Nataniel Ruiz, Yuanzhen Li, Varun Jampani, Yael Pritch, Michael Rubinstein, and Kfir Aberman.
\newblock Dreambooth: Fine tuning text-to-image diffusion models for subject-driven generation.
\newblock In \emph{Proceedings of the IEEE/CVF Conference on Computer Vision and Pattern Recognition}, pp.\  22500--22510, 2023.

\bibitem[Sakaguchi et~al.(2021)Sakaguchi, Bras, Bhagavatula, and Choi]{sakaguchi2021winogrande}
Keisuke Sakaguchi, Ronan~Le Bras, Chandra Bhagavatula, and Yejin Choi.
\newblock Wino{G}rande: An {A}dversarial {W}inograd {S}chema {C}hallenge at {S}cale.
\newblock \emph{Communications of the ACM}, 64\penalty0 (9):\penalty0 99--106, 2021.

\bibitem[Sap et~al.(2019)Sap, Rashkin, Chen, Bras, and Choi]{sap2019socialiqa}
Maarten Sap, Hannah Rashkin, Derek Chen, Ronan~Le Bras, and Yejin Choi.
\newblock {S}ocial{IQ}a: Commonsense {R}easoning about {S}ocial {I}nteractions.
\newblock \emph{CoRR}, abs/1904.09728, 2019.
\newblock URL \url{http://arxiv.org/abs/1904.09728}.

\bibitem[Tewel et~al.(2023)Tewel, Gal, Chechik, and Atzmon]{r1e}
Yoad Tewel, Rinon Gal, Gal Chechik, and Yuval Atzmon.
\newblock Key-locked rank one editing for text-to-image personalization.
\newblock In \emph{ACM SIGGRAPH 2023 Conference Proceedings}, pp.\  1--11, 2023.

\bibitem[Tieleman(2012)]{tieleman2012lecture}
Tijmen Tieleman.
\newblock Lecture 6.5-rmsprop: Divide the gradient by a running average of its recent magnitude.
\newblock \emph{COURSERA: Neural networks for machine learning}, 4\penalty0 (2):\penalty0 26, 2012.

\bibitem[Touvron et~al.(2023{\natexlab{a}})Touvron, Lavril, Izacard, Martinet, Lachaux, Lacroix, Rozi{\`{e}}re, Goyal, Hambro, Azhar, Rodriguez, Joulin, Grave, and Lample]{touvron2023open}
Hugo Touvron, Thibaut Lavril, Gautier Izacard, Xavier Martinet, Marie{-}Anne Lachaux, Timoth{\'{e}}e Lacroix, Baptiste Rozi{\`{e}}re, Naman Goyal, Eric Hambro, Faisal Azhar, Aur{\'{e}}lien Rodriguez, Armand Joulin, Edouard Grave, and Guillaume Lample.
\newblock {LL}a{MA}: Open and efficient foundation language models.
\newblock \emph{CoRR}, abs/2302.13971, 2023{\natexlab{a}}.
\newblock \doi{10.48550/ARXIV.2302.13971}.
\newblock URL \url{https://doi.org/10.48550/arXiv.2302.13971}.

\bibitem[Touvron et~al.(2023{\natexlab{b}})Touvron, Martin, Stone, Albert, Almahairi, Babaei, Bashlykov, Batra, Bhargava, Bhosale, Bikel, Blecher, Canton{-}Ferrer, Chen, Cucurull, Esiobu, Fernandes, Fu, Fu, Fuller, Gao, Goswami, Goyal, Hartshorn, Hosseini, Hou, Inan, Kardas, Kerkez, Khabsa, Kloumann, Korenev, Koura, Lachaux, Lavril, Lee, Liskovich, Lu, Mao, Martinet, Mihaylov, Mishra, Molybog, Nie, Poulton, Reizenstein, Rungta, Saladi, Schelten, Silva, Smith, Subramanian, Tan, Tang, Taylor, Williams, Kuan, Xu, Yan, Zarov, Zhang, Fan, Kambadur, Narang, Rodriguez, Stojnic, Edunov, and Scialom]{touvron2019open}
Hugo Touvron, Louis Martin, Kevin Stone, Peter Albert, Amjad Almahairi, Yasmine Babaei, Nikolay Bashlykov, Soumya Batra, Prajjwal Bhargava, Shruti Bhosale, Dan Bikel, Lukas Blecher, Cristian Canton{-}Ferrer, Moya Chen, Guillem Cucurull, David Esiobu, Jude Fernandes, Jeremy Fu, Wenyin Fu, Brian Fuller, Cynthia Gao, Vedanuj Goswami, Naman Goyal, Anthony Hartshorn, Saghar Hosseini, Rui Hou, Hakan Inan, Marcin Kardas, Viktor Kerkez, Madian Khabsa, Isabel Kloumann, Artem Korenev, Punit~Singh Koura, Marie{-}Anne Lachaux, Thibaut Lavril, Jenya Lee, Diana Liskovich, Yinghai Lu, Yuning Mao, Xavier Martinet, Todor Mihaylov, Pushkar Mishra, Igor Molybog, Yixin Nie, Andrew Poulton, Jeremy Reizenstein, Rashi Rungta, Kalyan Saladi, Alan Schelten, Ruan Silva, Eric~Michael Smith, Ranjan Subramanian, Xiaoqing~Ellen Tan, Binh Tang, Ross Taylor, Adina Williams, Jian~Xiang Kuan, Puxin Xu, Zheng Yan, Iliyan Zarov, Yuchen Zhang, Angela Fan, Melanie Kambadur, Sharan Narang, Aur{\'{e}}lien Rodriguez, Robert Stojnic, Sergey Edunov,
  and Thomas Scialom.
\newblock Llama 2: Open foundation and fine-tuned chat models.
\newblock \emph{CoRR}, abs/2307.09288, 2023{\natexlab{b}}.
\newblock \doi{10.48550/ARXIV.2307.09288}.
\newblock URL \url{https://doi.org/10.48550/arXiv.2307.09288}.

\bibitem[Trendafilov \& Gallo(2021)Trendafilov and Gallo]{trendafilov2021multivariate}
Nickolay Trendafilov and Michele Gallo.
\newblock \emph{Multivariate data analysis on matrix manifolds}.
\newblock Springer, 2021.

\bibitem[Vandereycken(2013)]{vandereycken2013low}
Bart Vandereycken.
\newblock Low-rank matrix completion by {R}iemannian optimization.
\newblock \emph{SIAM Journal on Optimization}, 23\penalty0 (2):\penalty0 1214--1236, 2013.

\bibitem[Wang et~al.(2025)Wang, Li, Wang, Chen, and Chen]{wang2024milora}
Hanqing Wang, Yixia Li, Shuo Wang, Guanhua Chen, and Yun Chen.
\newblock {M}i{L}o{RA}: Harnessing minor singular components for parameter-efficient {LLM} finetuning.
\newblock In Luis Chiruzzo, Alan Ritter, and Lu~Wang (eds.), \emph{Proceedings of the 2025 Conference of the Nations of the Americas Chapter of the Association for Computational Linguistics: Human Language Technologies, {NAACL} 2025 - Volume 1: Long Papers, Albuquerque, New Mexico, USA, April 29 - May 4, 2025}, pp.\  4823--4836. Association for Computational Linguistics, 2025.
\newblock \doi{10.18653/V1/2025.NAACL-LONG.248}.
\newblock URL \url{https://doi.org/10.18653/v1/2025.naacl-long.248}.

\bibitem[Wang et~al.(2024)Wang, Yu, and Li]{wang2024lora}
Shaowen Wang, Linxi Yu, and Jian Li.
\newblock Lo{RA}-{GA}: Low-rank adaptation with gradient approximation.
\newblock \emph{Advances in Neural Information Processing Systems}, 37:\penalty0 54905--54931, 2024.

\bibitem[Wisdom et~al.(2016)Wisdom, Powers, Hershey, Le~Roux, and Atlas]{wisdom2016full}
Scott Wisdom, Thomas Powers, John Hershey, Jonathan Le~Roux, and Les Atlas.
\newblock Full-{C}apacity {U}nitary {R}ecurrent {N}eural {N}etworks.
\newblock \emph{Advances in neural information processing systems}, 29, 2016.

\bibitem[Yang et~al.(2024)Yang, Li, Zhou, Song, Wu, Nie, and Ghanem]{yang2024corda}
Yibo Yang, Xiaojie Li, Zhongzhu Zhou, Shuaiwen Song, Jianlong Wu, Liqiang Nie, and Bernard Ghanem.
\newblock Cor{DA}: Context-oriented decomposition adaptation of large language models for task-aware parameter-efficient fine-tuning.
\newblock \emph{Advances in Neural Information Processing Systems}, 37:\penalty0 71768--71791, 2024.

\bibitem[Yen et~al.(2024)Yen, Si, Meng, Yu, Duvvuri, Dhillon, Hsieh, and Kumar]{yen2024lora}
Jui-Nan Yen, Si~Si, Zhao Meng, Felix Yu, Sai~Surya Duvvuri, Inderjit~S Dhillon, Cho-Jui Hsieh, and Sanjiv Kumar.
\newblock Lora done rite: Robust invariant transformation equilibration for lora optimization.
\newblock \emph{arXiv preprint arXiv:2410.20625}, 2024.

\bibitem[Zellers et~al.(2019)Zellers, Holtzman, Bisk, Farhadi, and Choi]{zellers2019hellaswag}
Rowan Zellers, Ari Holtzman, Yonatan Bisk, Ali Farhadi, and Yejin Choi.
\newblock {H}ella{S}wag: Can a machine really finish your sentence?
\newblock In Anna Korhonen, David~R. Traum, and Llu{\'{\i}}s M{\`{a}}rquez (eds.), \emph{Proceedings of the 57th Conference of the Association for Computational Linguistics, {ACL} 2019, Florence, Italy, July 28- August 2, 2019, Volume 1: Long Papers}, pp.\  4791--4800. Association for Computational Linguistics, 2019.
\newblock \doi{10.18653/V1/P19-1472}.
\newblock URL \url{https://doi.org/10.18653/v1/p19-1472}.

\bibitem[Zhang \& Pilanci(2024)Zhang and Pilanci]{zhang2024riemannian}
Fangzhao Zhang and Mert Pilanci.
\newblock Riemannian preconditioned lo{RA} for fine-tuning foundation models.
\newblock In \emph{Proceedings of the 41st International Conference on Machine Learning}, ICML'24. JMLR.org, 2024.

\end{thebibliography}
\bibliographystyle{iclr2026_conference}

\newpage
\appendix

\section{Inital point search}%
\label{sec:inital_point_search}
In this section we introduce the proof of Theorem \ref{theorem:optimal_initialization}
	
	\begin{proof}
		In order to represent the optimization task (\ref{eq:inital_opt_task}) in a more simple way, 
		we will use a slightly different formula for the projection of the full loss gradient onto the tangent space of 
		the fixed-rank manifold. Let 
		\begin{equation*}
			\Delta W = A_L B^\top = A B_R^\top \in \mathcal{M}_r,
		\end{equation*} 
		and the tangent space is parametrized as follows
		\begin{equation*}
			\mathcal{T}_{\Delta W} \mathcal{M}_r = 
			\{ \dot A B_R^\top + A_L \dot B^\top \mid \dot B \in \mathbb{R}^{n \times r}, \dot A \in \mathbb{R}^{m \times r}, A_L^\top \dot A = 0 \},
		\end{equation*}
		
		then one may derive an orthogonal projection formula for any matrix $Z$ (see, e.g. \citep[eq.~7.53]{boumal2023introduction})
		\begin{equation*}
			P_{\mathcal{T}_{\Delta W} \mathcal{M}_r}Z = Z - (I - A_L A_L^\top)Z(I - B_R B_R^\top) \in \mathcal{T}_{\Delta W} \mathcal{M}_r.
		\end{equation*}
		As the tangent space is a linear space, the operation of orthogonal projection 
		can be written as an optimization task
		\begin{equation}\label{eq:apndx_proj_as_opt}
			P_{\mathcal{T}_{\Delta W} \mathcal{M}_r}Z = \argmin_{\xi \in \mathcal{T}_{\Delta W} \mathcal{M}_r} \|Z - \xi \|^2_F.
		\end{equation}
		Since

        \begin{equation*}
            \begin{aligned}
    			\|\nabla \mathcal{L}\|^2_F & = 
                \| \left(\nabla \mathcal{L} - P_{\mathcal{T}_{\Delta W} \mathcal{M}_r}\nabla \mathcal{L} \right) + P_{\mathcal{T}_{\Delta W} \mathcal{M}_r}\nabla \mathcal{L} \|^2_F  = \\
    			& = \|\nabla \mathcal{L} - P_{\mathcal{T}_{\Delta W} \mathcal{M}_r}\nabla \mathcal{L}\|^2_F + 
    			\|P_{\mathcal{T}_{\Delta W} \mathcal{M}_r}\nabla \mathcal{L}\|_F^2 + 2 \langle \nabla \mathcal{L} - P_{\mathcal{T}_{\Delta W} \mathcal{M}_r}\nabla \mathcal{L}, P_{\mathcal{T}_{\Delta W} \mathcal{M}_r} \nabla \mathcal{L} \rangle \\
                & = \|\nabla \mathcal{L} - P_{\mathcal{T}_{\Delta W} \mathcal{M}_r}\nabla \mathcal{L}\|^2_F + 
    			\|P_{\mathcal{T}_{\Delta W} \mathcal{M}_r}\nabla \mathcal{L}\|_F^2 + 2 \langle P_{\mathcal{T}_{\Delta W} \mathcal{M}_r} \left( \nabla \mathcal{L} - P_{\mathcal{T}_{\Delta W} \mathcal{M}_r}\nabla \mathcal{L} \right), \nabla \mathcal{L} \rangle \\
                &= \|\nabla \mathcal{L} - P_{\mathcal{T}_{\Delta W} \mathcal{M}_r}\nabla \mathcal{L}\|^2_F + 
    			\|P_{\mathcal{T}_{\Delta W} \mathcal{M}_r}\nabla \mathcal{L}\|_F^2 + 2 \langle \underbrace{P_{\mathcal{T}_{\Delta W} \mathcal{M}_r}\nabla \mathcal{L} - P_{\mathcal{T}_{\Delta W} \mathcal{M}_r}\nabla \mathcal{L}}_{=0}, \nabla \mathcal{L} \rangle \\
                &= \|\nabla \mathcal{L} - P_{\mathcal{T}_{\Delta W} \mathcal{M}_r}\nabla \mathcal{L}\|^2_F + 
    			\|P_{\mathcal{T}_{\Delta W} \mathcal{M}_r}\nabla \mathcal{L}\|_F^2
            \end{aligned}
		\end{equation*}
        
		so the optimization task (\ref{eq:inital_opt_task}) is equivalent to 
		\begin{equation*}
			\|\nabla \mathcal{L} - P_{\mathcal{T}_{\Delta W} \mathcal{M}_r}\nabla \mathcal{L}\|_F^2 \to \min_{\Delta W},
		\end{equation*}
		which, in turn, using (\ref{eq:apndx_proj_as_opt}) is equivalent to the task
		\begin{equation*}
			\min_{{\Delta W}}\min_{\xi \in \mathcal{T}_{\Delta W} \mathcal{M}_r}\|\nabla \mathcal{L} - \xi\|^2_F .
		\end{equation*}

		Due to the fact that every vector of the tangent space is an element of a set of 
		all matrices with $2r$-bounded rank $\mathcal{M}_{ \le 2r}$, one may use 
		the Eckart-Young-Mirsky theorem (see \citep{eckart1936approximation}) to obtain the following lower bound: 
	\begin{equation}\label{eq:apndx_lower_bound}
				 \|\nabla \mathcal{L} - P_{\mathcal{T}_{\Delta W} \mathcal{M}_r}\nabla \mathcal{L}\|_F^2 
				\ge \| \nabla \mathcal{L} - \mathrm{truncSVD}\, (\nabla \mathcal{L}, 2r)\|_F, 
				\forall \; \Delta W \in \mathcal{M}_r, 
				\xi \in \mathcal{T}_{\Delta W} \mathcal{M}_r.
		\end{equation}
		But it is possible to ensure $\Delta W^{(0)}_* \in \mathcal{M}_r$ and 
		$\xi_* \in \mathcal{T}_{\Delta W^{(0)}_*} \mathcal{M}_r$ which turn the inequality (\ref{eq:apndx_lower_bound}) into 
		the equality. One may take 
		\begin{equation*}
			\begin{split}
				&\Delta W^{(0)}_* =  A_L (\alpha B)^\top = \alpha U_{1,r}V_{r,2r}^\top = (\alpha A) B_R^\top, \quad \alpha \in \mathbb{R},\\
				&\xi_* = \dot{A}B_R^\top + A_L \dot{B}^\top =
				\left(U_{r,2r}\Sigma_{r,2r}\right)B_R^\top + 
				A_L\left(\Sigma_{1,r}V_{1,r}\right)^\top = \mathrm{truncSVD}\,(\nabla \mathcal{L}, 2r).
			\end{split}
		\end{equation*}
		Note, that $\dot{A}^\top A_L = 0$. So the initialization $\Delta W^{(0)}_*$ ensures that $\mathrm{truncSVD}\,(\nabla \mathcal{L}, 2r)$ lies in the 
		tangent space $\mathcal{T}_{\Delta W^{(0)}_*} \mathcal{M}_r$, which means that 
		the first step of the RiemannLoRA Algorithm \ref{alg:fr_lora_hb} will get the Riemannian gradient
		equal to $\mathrm{truncSVD}\,(\nabla \mathcal{L}, 2r)$.

        To generalize the result, we should change the parametrization of the tangent space. Because of the fact that each of the tangent vectors $\xi$ lies in $\mathcal{M}_{\leq 2r}$,
        we may use an unconstrained skeleton decomposition (without constraints for
        $\dot{A}$ and $A_L$):
        \begin{equation*}
            \begin{split}
                & \xi_* = \begin{bmatrix} A_*, \dot A_* \end{bmatrix} \begin{bmatrix} \dot B_*^\top \\ B_*^\top \end{bmatrix} = \begin{bmatrix} A_*, \dot A_* \end{bmatrix} \, I_{2r} \, \begin{bmatrix} \dot B_*^\top \\ B_*^\top \end{bmatrix} = \\
                & = \begin{bmatrix} A_*, \dot A_* \end{bmatrix} \, S \, S^{-1} \begin{bmatrix} \dot B_*^\top \\ B_*^\top \end{bmatrix}, \quad  S \in \mathbb{R}^{2r \times 2r}, \; \det S \ne 0.
            \end{split}
        \end{equation*} 

        Representing $S$ and its inverse as  block matrices: 
        \begin{equation*}
            S = \begin{bmatrix} S_{11} & S_{12} \\ S_{21} & S_{22} \end{bmatrix}, \quad S^{-1} = \begin{bmatrix} C_{11} & C_{12} \\ C_{21} & C_{22} \end{bmatrix},
        \end{equation*}
        one arrives to (\ref{eq:special_init_form}):
        \begin{equation*}
            \begin{split}
                &\xi_* = \begin{bmatrix} A_*, \dot A_* \end{bmatrix}  
                \begin{bmatrix} S_{11} & S_{12} \\ S_{21} & S_{22} \end{bmatrix} \begin{bmatrix} C_{11} & C_{12} \\ C_{21} & C_{22} \end{bmatrix}
                \begin{bmatrix} \dot {B_*} & B_* \end{bmatrix}^\top, \\
                &\Delta W^{(0)}_* = \begin{bmatrix} A'_*, \dot A'_* \end{bmatrix}
                \begin{bmatrix} S_{11} \\ S_{21} \end{bmatrix}
                \begin{bmatrix} C_{21} & C_{22} \end{bmatrix}
                \begin{bmatrix} {\dot{B}_*'} & {{B'}_*} \end{bmatrix}^\top.
            \end{split}
        \end{equation*} 
        To derive the version that is utilized in practice,
        one may take $S$ to be block-diagonal with $S_{11} = \alpha I_{r}, S_{22} = I_r, \alpha \in \mathbb{R}$.
	\end{proof} 

\section{Riemannion with LOI}
    The complete framework is summarized in Algorithm~\ref{alg:full_fr_riemannion}, which consists of an LOI initialization step followed by \texttt{max\_iters} iterations of the Riemannion optimizer.

    Specifically, the first $3$ lines 
	are dedicated to the computation of LOI via 
	\texttt{BackPropRSVD} for given layer weights and oversampling parameter $p$, that 
	increases the accuracy of the singular approximation. 
    The $4$-th step 
	initializes Heavy-Ball momentum matrices. Therefore, the  asymptotical complexity of this stage is determined by the complexity of Algorithm~\ref{alg:bpr_svd}.
   \begin{algorithm}[h!tp]
		\caption{\texttt{Riemannion with LOI}}
		\label{alg:full_fr_riemannion}
		\textbf{Require:}
			Weights $W \in \mathbb{R}^{m \times n}$, rank $r \in \mathbb{N}$,
			step size $\eta$,
			momentum coefficient $\beta$,
            weight decay coefficient $\gamma$,
			oversampling parameter $p$,
			power-step parameter $q$. \\
		\textbf{Ensure:}
			Tuning parameters $\Delta W^* \in \mathcal{M}_r$.\\
		\textbf{Function:}
		\begin{algorithmic}[1]
			\STATE $A_L, \_, B^\top := \mathtt{BackPropRSVD}(2r, p, q, \mathcal{L}, W)$.
			\hfill\COMMENT{$\mathcal{O}\left((m + n)r^2\right)$}

			\STATE $A_L, B^\top := A_L[\colon, \colon r], \; B[\colon, r\colon ]^\top$.

			\STATE $W' := W - A_L B^\top$. 
			\hfill\COMMENT{$\mathcal{O}\left(mn\right)$}

			\STATE $A_{\mathrm{HB}}, B_{\mathrm{HB}} := 0, 0$.

			\FOR {$i := 0, \dots, \mathtt{max \_ iters}$}
				\STATE $B_R := \mathtt{qr}\left(B\right).\mathtt{Q}.$ 
				\hfill\COMMENT{$\mathcal{O}\left(nr^2\right)$}

				\STATE $\dot{A} := \nabla_{Z_1} 
							\mathcal{L}\left( W' + Z_1B_R^\top + 
							A_L Z_2^\top\right)|_{Z_1=0, Z_2=B}.$
                \hfill\COMMENT{$ = \nabla_{W}\mathcal{L}(W)B_R$}
				\STATE $\dot{B} := \nabla_{Z_2} 
							\mathcal{L}\left(W'+ Z_1B_R^\top + 
							A_L Z_2^\top\right)|_{Z_1=0, Z_2=B}.$
                \hfill\COMMENT{$ = \nabla_{W}\mathcal{L}(W)^{\top}A_L$}
				\STATE $\dot{A}_{\mathrm{prev}}, \dot{B}_{\mathrm{prev}} := 
				\mathtt{ProjectLR}\left((A_{\mathrm{M}}, B_{\mathrm{M}}), A_L, B_R\right)$. 
				\hfill\COMMENT{$\mathcal{O}\left((m + n)r^2\right)$}

				\STATE $\dot{A}, \dot{B} := \beta \dot{A}_{\mathrm{prev}} + 
				(I - A_L A_L^\top) \dot{A}, \;
                \beta \dot{B}_{\mathrm{prev}} + 
				\dot{B}$ 
				\hfill\COMMENT{$\mathcal{O}\left((n+m)r^2\right)$}

                \STATE $\dot A, \dot B := \mathtt{ProjectLR}\left(\mathtt{OrthoLR} \left(A_L, B_R, \dot A, \dot B, r, \right), A_L, B_R\right)$
                \hfill \COMMENT{$\mathcal{O}\left((m + n)r^2 +r^3\right)$}

				\STATE 
				$
					U, \Sigma, V^\top := 
					\mathtt{RetractionLR}\left( \left[-\eta\dot{A},  A_L\right],
					\left[  B_R, -\eta(\dot{B} + \gamma B_R)\right] \right) 
				$ \hfill \COMMENT{$\mathcal{O}\left((m + n)r^2 +r^3\right)$}

				\STATE $A_{\mathrm{HB}}, B_{\mathrm{HB}} := [\dot{A}, A_L], \; [B_R, \dot{B}]$ 

				\STATE $A_L, B := U, \Sigma V^\top$ 
				\hfill\COMMENT{$\mathcal{O}\left(nr^2\right)$}

			\ENDFOR	

			\RETURN $A_L, B$.
		\end{algorithmic}
	\end{algorithm}

\section{Riemann-SGD Algorithm}
In Algorithm \ref{alg:fr_lora_hb}, we present the RiemannSGD version of the fine-tuning algorithm. The first steps are dedicated to the computation of the optimal initialization using \texttt{BackPropRSVD}(Algorithm \ref{alg:bpr_svd}). After initialization, the algorithm follows an Adam-like procedure adapted to the Riemannian setting: it maintains exponentially smoothed estimates of momentum and gradient norms ($4th$ step). The overall cycle encapsulates the computation of the Riemannian gradient at the current point(steps $7-8$), the update of the adaptive momentum terms via  \texttt{ProjectLR}(Algorithm \ref{alg:transport}), and the retraction step that projects the updated tangent direction back onto the manifold(steps $15-17$).
	This approach
	has asymptotical complexity of $\mathcal{O}((m + n)r^2 +r^3)$
	and $(2(q + 1) + \mathtt{max\_iters})$ amount of backward calls.
\begin{algorithm}[!h]
		\caption{\texttt{RiemannSGD with LOI}}
		\label{alg:fr_lora_hb}
		\textbf{Require:}
			Weights $W \in \mathbb{R}^{m \times n}$, rank $r \in \mathbb{N}$,
			step size $\eta$,
			momentum coefficient $\beta$,
			oversampling parameter $p$,
			power-step parameter $q$,
			\texttt{simulate\_Adam},
			Adam momentum coefficient $\gamma$.\\
		\textbf{Ensure:}
			Tuning parameters $\Delta W^* \in \mathcal{M}_r$.\\
		\textbf{Function:}
		\begin{algorithmic}[1]
			\STATE $A_L, \_, B^\top := \mathtt{BackPropRSVD}(2r, p, q, \mathcal{L}, W)$.
			\hfill\COMMENT{$\mathcal{O}\left((m + n)r^2\right)$}

			\STATE $A_L, B^\top := A_L[\colon, \colon r], \; B[\colon, r\colon ]^\top$.

			\STATE $W' := W - A_L B^\top$. 
			\hfill\COMMENT{$\mathcal{O}\left(mn\right)$}

			\STATE $A_{\mathrm{HB}}, B_{\mathrm{HB}}, S_A, S_B := 0$.

			\FOR {$i := 0, \dots, \mathtt{max \_ iters}$}

				\STATE $B_R := \mathtt{qr}\left(B\right).\mathtt{Q}.$ 
				\hfill\COMMENT{$\mathcal{O}\left(nr^2\right)$}

				\STATE $\dot{A} := \nabla_{Z_1} 
							\mathcal{L}\left( W + Z_1B_R^\top + 
							A_L Z_2^\top\right)|_{Z_1=0, Z_2=B}.$
				\STATE $\dot{B} := \nabla_{Z_2} 
							\mathcal{L}\left( W + Z_1B_R^\top + 
							A_L Z_2^\top\right)|_{Z_1=0, Z_2=B}.$

				\STATE $\dot{A}_{\mathrm{prev}}, \dot{B}_{\mathrm{prev}} := 
				\mathtt{ProjectLR}\left((A_{\mathrm{HB}}, B_{\mathrm{HB}}), A_L, B_R\right).$ 
				\hfill\COMMENT{$\mathcal{O}\left((m + n)r^2\right)$}

				\STATE $\dot{A}, \dot{B} := \beta \dot{A}_{\mathrm{prev}} + 
				(1 - \beta)(I - A_L A_L^\top) \dot{A}, \;
                \beta \dot{B}_{\mathrm{prev}} + 
				(1 - \beta)\dot{B}$ 
				\hfill\COMMENT{$\mathcal{O}\left(nr^2\right)$}

				\IF {$\mathtt{simulate\_Adam}$}

					\STATE $S_A, S_B := \gamma \|\dot{A}\|_F + (1 - \gamma)S_A,\;
                    \gamma \|\dot{B}\|_F + (1 - \gamma)S_B$
					\hfill\COMMENT{$\mathcal{O}\left((m+n)r\right)$}

					\STATE $\dot{A}, \dot{B} := \dot{A}/S_A, \dot{B}/S_B$
					\hfill\COMMENT{$\mathcal{O}\left((m + n)r\right)$}

				\ENDIF

				\STATE 
				$
					U, \Sigma, V^\top := 
					\mathtt{RetractionLR}\left( \left[\eta\dot{A},  A_L\right],
					\left[  B_R, \eta\dot{B} + B\right] \right) 
				$ \hfill \COMMENT{$\mathcal{O}\left((m + n)r^2 +r^3\right)$}

				\STATE $A_{\mathrm{HB}}, B_{\mathrm{HB}} := [\dot{A}, A_L], \; [B_R, \dot{B}]$ 

				\STATE $A_L, B := U, \Sigma V^\top$ 
				\hfill\COMMENT{$\mathcal{O}\left(nr^2\right)$}

			\ENDFOR	

			\RETURN $A_L, B$.
		\end{algorithmic}
	\end{algorithm}

\begin{table}[h!tp]
	\centering
    \caption{The average accuracy (in \%) among $8$ tasks of fine-tuned Llama 3.2-1b   
	using different SGD variants, tested on 
    Commonsense Reasoning benchmark. The <<R->> prefix stands for Riemannian, 
    the postfix <<-LOI>> stands for locally optimal initialization. LoRA rank is set to 16. 
	In all experiments excpet for <<-LOI>> the $A$ factor in $AB^\top$ has orthonormal columns and $B$ is zero.}
	\scriptsize
	\setlength{\tabcolsep}{3pt}
    \begin{tabular}{lrrrrrrrrr}
        \toprule
        Task & BoolQ & PIQA & SIQA & hella-& wino-& ARC-E & ARC-C & OBQA & All \\
        Initialization &  &  &  &  swag &  grande &  &  & &  \\
        \midrule
        Raw & $40.1$ & $55.4$ & $50.3$ & $25.8$ & $50.0$ & $61.9$ & $41.8$ & $42.8$ & $46.0$ \\
        LoRA & $64.0_{\pm 0.2}$ & $75.3_{\pm 0.3}$ & $69.6_{\pm 0.4}$ & $82.2_{\pm 0.1}$ & $53.0_{\pm 1.0}$ & $75.4_{\pm 1.7}$ & $56.5_{\pm 0.6}$ & $67.1_{\pm 2.3}$ & $67.9_{ \pm 0.4}$ \\
        LoRA-LOI & $64.9_{\pm 0.4}$ & $76.8_{\pm 0.3}$ & $71.2_{\pm 2.2}$ & $84.1_{\pm 0.1}$ & $56.7_{\pm 0.6}$ & $77.0_{\pm 4.2}$ & $61.1_{\pm 3.6}$ & $68.9_{\pm 3.3}$ & $70.1_{ \pm 1.6}$ \\
        RSLoRA & $64.5_{\pm 0.2}$ & $77.2_{\pm 0.3}$ & $68.1_{\pm 0.8}$ & $86.3_{\pm 0.1}$ & $58.2_{\pm 0.5}$ & $76.5_{\pm 1.5}$ & $58.6_{\pm 2.9}$ & $70.0_{\pm 1.9}$ & $70.0_{ \pm 0.8}$ \\
        RiemannLoRA & $\mathbf{65.1}_{\pm 0.4}$ & $77.7_{\pm 0.3}$ & $69.1_{\pm 1.5}$ & $85.8_{\pm 0.2}$ & $58.5_{\pm 0.4}$ & $74.4_{\pm 1.9}$ & $56.9_{\pm 1.3}$ & $69.3_{\pm 2.5}$ & $69.6_{ \pm 0.9 }$ \\
        RiemannLoRA-LOI & $\mathbf{65.2}_{\pm 0.4}$ & $\mathbf{79.4}_{\pm 0.4}$ & $\mathbf{75.6}_{\pm 0.4}$ & $\mathbf{87.3}_{\pm 0.1}$ & $\mathbf{62.4}_{\pm 0.4}$ & $\mathbf{79.9}_{\pm 0.8}$ & $\mathbf{63.6}_{\pm 1.9}$ & $\mathbf{73.8}_{\pm 0.5}$ & $\mathbf{73.4}_{  \pm 0.3 }$ \\
        \bottomrule
    \end{tabular}
    \label{table:sgd_comp}
\end{table}

\section{LMO bounds}\label{sec:lmo_bounds}

Let $A \in \mathcal{T}_W \mathcal{M}_r$ be a vector from the tangent space. Consider the problem
\begin{equation}
\begin{split}\label{eq:task1}
& \langle A, X \rangle \to \max_X,\\
& \text{s.t.} \quad \|X\|_2 \le 1,
\end{split}
\end{equation}
whose (any) solution will be denoted by $X_\ast = \mathtt{Ortho}(A)$. Denote $f_\ast = \langle A, X_\ast \rangle$ the optimal value.

Consider the restricted (LMO) problem
\begin{equation}
\begin{split}\label{eq:task2}
& \langle A, X \rangle \to \max_X,\\
& \text{s.t.} \quad \|X\|_2 \le 1,\\
&  X \in \mathcal{T}_W \mathcal{M}_r.
\end{split}
\end{equation}
Let $X^\ast$ denote an optimizer of the task derived in (\ref{eq:task2}) and $f^\ast = \langle A, X^\ast \rangle$ its optimal value.

\begin{proposition}\label{prop:bound}
Assume $A\in\mathcal T_W\mathcal M_r$ and let $P$ be the orthogonal projector onto $\mathcal T_W\mathcal M_r$. If
\[
\|X_\ast - P X_\ast\|_2 \le \varepsilon
\]
for some $\varepsilon \ge 0$, then
\[
\frac{1}{1+\varepsilon}\, f_\ast \le f^\ast \le f_\ast.
\]
\end{proposition}

\begin{proof}
Since $A\in\mathcal T_W\mathcal M_r$, we have $PA=A$ and $P=P^\top$. By assumption $\|X_\ast-PX_\ast\|_2\le\varepsilon$ and, because $\|X_\ast\|_2=1$, the reverse triangle inequality yields
\[
\big|\|X_\ast\|_2 - \|P X_\ast\|_2\big| \le \varepsilon,
\]
hence
\[
1-\varepsilon \le \|P X_\ast\|_2 \le 1+\varepsilon.
\]

Define the normalized matrix
\[
Y := \frac{P X_\ast}{\|P X_\ast\|_2}\in\mathcal T_W\mathcal M_r,
\qquad \|Y\|_2=1.
\]
By definition of $f^\ast$ (the maximum of $\langle A,\cdot\rangle$ over $\{X\in\mathcal T_W\mathcal M_r:\|X\|_2\le1\}$) we have
\[
f^\ast \ge \langle A, Y\rangle
= \frac{\langle A, P X_\ast\rangle}{\|P X_\ast\|_2}
= \frac{\langle P A, X_\ast\rangle}{\|P X_\ast\|_2}
= \frac{\langle A, X_\ast\rangle}{\|P X_\ast\|_2}
= \frac{f_\ast}{\|P X_\ast\|_2}.
\]
Using the inequality $\|P X_\ast\|_2 \le 1+\varepsilon$ we obtain the claimed lower bound
\[
f^\ast \ge \frac{f_\ast}{1+\varepsilon}.
\]

The upper bound $f^\ast \le f_\ast$ is trivial because the feasible set of the task derived in (\ref{eq:task2}) is a subset of the feasible set of the problem described in (\ref{eq:task1}). Combining the two inequalities gives
\[
\frac{1}{1+\varepsilon}\, f_\ast \le f^\ast \le f_\ast,
\]
which completes the proof.
\end{proof}

\begin{proposition}\label{prop:eps_le_one}
(Upper bound on the projection residual.) With the notation above, the residual
\[
\varepsilon := \|X_\ast - P X_\ast\|_2
\]
satisfies $\varepsilon \le 1$.
\end{proposition}

\begin{proof}
Write the thin SVD of the base point $W$ as $W = U_r \Sigma_r V_r^\top$, where $U_r\in\mathtt{St}(m,r)$ and $V_r\in\mathtt{St}(n,r)$ have $r$ orthonormal columns. Let $U_\perp$ and $V_\perp$ be orthonormal complements so that $[U_r\;U_\perp]\in\mathbb{R}^{m\times 2r}$ and $[V_r\;V_\perp]\in\mathbb{R}^{n\times 2r}$ are orthogonal matrices (here we assume $m,n\ge 2r$; if $m$ or $n$ is smaller the argument is adapted in the obvious way).

Because $A\in\mathcal T_W\mathcal M_r$, one can write (using a suitable $2r\times 2r$ block representation)~\citep{vandereycken2013low}
\[
A = \begin{bmatrix} U_r & U_\perp \end{bmatrix}
\begin{bmatrix} \tilde{\mathcal{A}} & \tilde{\mathcal{B}} \\[4pt] \tilde{\mathcal{C}} & 0 \end{bmatrix}
\begin{bmatrix} V_r^\top \\[2pt] V_\perp^\top \end{bmatrix},
\]
where $\tilde{\mathcal{A}},\tilde{\mathcal{B}},\tilde{\mathcal{C}}\in\mathbb R^{r\times r}$ (the $(2,2)$ block vanishes for tangent-space elements of this form). The rank of $A$ is therefore at most $2r$.

Let $X_\ast=\mathtt{Ortho}(A)$ be a maximizer of (\ref{eq:task1}); represent $X_\ast$ in the same block basis,
\[
X_\ast = \begin{bmatrix} U_r & U_\perp \end{bmatrix}
\begin{bmatrix} \mathcal{A} & \mathcal{B} \\[4pt] \mathcal{C} & \mathcal{D} \end{bmatrix}
\begin{bmatrix} V_r^\top \\[2pt] V_\perp^\top \end{bmatrix},
\]
where the $2r\times 2r$ block matrix is orthogonal,
\begin{equation}\label{eq:ortho_g}
\begin{bmatrix} \mathcal{A} & \mathcal{B} \\[4pt] \mathcal{C} & \mathcal{D} \end{bmatrix}\in\mathcal O(2r),
\end{equation}
where $\mathcal{O}(2r)$ denotes the orthogonal group in $\mathbb{R}^{2r \times 2r}$.
By Lemma~\ref{lemma:proj_block} below we have
\[
P X_\ast = \begin{bmatrix} U_r & U_\perp \end{bmatrix}
\begin{bmatrix} \mathcal{A} & \mathcal{B} \\[4pt] \mathcal{C} & 0 \end{bmatrix}
\begin{bmatrix} V_r^\top \\[2pt] V_\perp^\top \end{bmatrix}.
\]
Therefore
\[
X_\ast - P X_\ast
= \begin{bmatrix} U_r & U_\perp \end{bmatrix}
\begin{bmatrix} 0 & 0 \\[4pt] 0 & \mathcal{D} \end{bmatrix}
\begin{bmatrix} V_r^\top \\[2pt] V_\perp^\top \end{bmatrix}
= U_\perp \mathcal{D} V_\perp^\top,
\]
and hence
\[
\|X_\ast - P X_\ast\|_2 = \|\mathcal{D}\|_2.
\]

Because the block matrix in (\ref{eq:ortho_g}) is orthogonal, its lower block
\[
\begin{bmatrix} \mathcal{C} \\[2pt] \mathcal{D} \end{bmatrix}\in\mathbb R^{2r\times r}
\]
has orthonormal columns. Thus $\big\|\begin{smallmatrix}\mathcal{C}\\[2pt]\mathcal{D}\end{smallmatrix}\big\|_2 = 1$, and in particular $\|\mathcal{D}\|_2 \le 1$. Consequently $\varepsilon = \|X_\ast - P X_\ast\|_2 \le 1$, as claimed.
\end{proof}

\begin{lemma}\label{lemma:proj_block}
Let
\[
W = U_r \Sigma_r V_r^\top, \quad U_r\in\mathtt{St}(m,r),\; V_r\in\mathtt{St}(n,r),
\]
and write an arbitrary matrix $X\in\mathbb R^{m\times n}$ in the block basis determined by $[U_r\;U_\perp]$ and $[V_r\;V_\perp]$,
\[
X = \begin{bmatrix} U_r & U_\perp \end{bmatrix}
\begin{bmatrix} \mathcal{A} & \mathcal{B} \\[4pt] \mathcal{C} & \mathcal{D} \end{bmatrix}
\begin{bmatrix} V_r^\top \\[2pt] V_\perp^\top \end{bmatrix}.
\]
Then the orthogonal projector $P$ onto $\mathcal T_W\mathcal M_r$ acts as
\[
P X = \begin{bmatrix} U_r & U_\perp \end{bmatrix}
\begin{bmatrix} \mathcal{A} & \mathcal{B} \\[4pt] \mathcal{C} & 0 \end{bmatrix}
\begin{bmatrix} V_r^\top \\[2pt] V_\perp^\top \end{bmatrix}.
\]
\end{lemma}

\begin{proof}
One convenient expression for the orthogonal projector onto the tangent space at $W=U_r\Sigma_r V_r^\top$ is
\[
P(\cdot) = (\cdot) - (I - U_r U_r^\top)\,(\cdot)\,(I - V_r V_r^\top).
\]
Applying this to $X$ written in the chosen block basis yields
\[
\begin{aligned}
P X
&= X - (I-U_r U_r^\top) X (I - V_r V_r^\top) = \\
&= \begin{bmatrix} U_r & U_\perp \end{bmatrix}
\begin{bmatrix} \mathcal{A} & \mathcal{B} \\[4pt] \mathcal{C} & \mathcal{D} \end{bmatrix}
\begin{bmatrix} V_r^\top \\[2pt] V_\perp^\top \end{bmatrix}
- \begin{bmatrix} 0 & U_\perp \end{bmatrix}
\begin{bmatrix} \mathcal{A} & \mathcal{B} \\[4pt] \mathcal{C} & \mathcal{D} \end{bmatrix}
\begin{bmatrix} 0 \\[2pt] V_\perp^\top \end{bmatrix} \\
&= \begin{bmatrix} U_r & U_\perp \end{bmatrix}
\left(
\begin{bmatrix} \mathcal{A} & \mathcal{B} \\[4pt] \mathcal{C} & \mathcal{D} \end{bmatrix}
-
\begin{bmatrix} 0 & 0 \\[4pt] 0 & \mathcal{D} \end{bmatrix}
\right)
\begin{bmatrix} V_r^\top \\[2pt] V_\perp^\top \end{bmatrix} \\
&= \begin{bmatrix} U_r & U_\perp \end{bmatrix}
\begin{bmatrix} \mathcal{A} & \mathcal{B} \\[4pt] \mathcal{C} & 0 \end{bmatrix}
\begin{bmatrix} V_r^\top \\[2pt] V_\perp^\top \end{bmatrix},
\end{aligned}
\]
which proves the lemma.
\end{proof}

\begin{remark}
In our experiments we observed that the residual $\varepsilon$ is typically small (in our reported runs it did not exceed $0.1$).
\end{remark}

\section{Dataset}%
\label{sec:dataset}
\subsection{Commonsense reasoning}%
\label{sub:commonsense_reasoning}
Following the approach of \citet{hu2023llm}, we combine the training datasets from all 8 tasks to form the final training set and evaluate performance on each task's individual test dataset.
        We adopt queries structure from \citet{hu2023llm} to Llama 3.2 instruct template. The prompt structure is demonstrated in Table \ref{table:csr_prompts}.
	\begin{table}[h!tp]
		\centering
		\caption{The structure of the queries for the Commonsense reasoning dataset}
		\small
		\begin{tabular}{lll}
			\toprule
			Task 		& Role & Fine-tuning Data Template \\
			\midrule
                BoolQ       & system & Please answer the following question with True or False. \\
                            &      & Follow the answer format, full answer not needed. \\
                            & user & Question:  [QUESTION] \\
                            &      & Answer format: True/False \\
                            & assistant & The correct answer is [ANSWER] \\
                \midrule
                PIQA       & system & Please choose the correct solution to the question. \\
                            &      & Follow the answer format, full answer not needed. \\
                            & user & Question:  [QUESTION] \\
                            &      & Solution1: [SOLUTION\_1] \\
                            &      & Solution2: [SOLUTION\_2] \\
                            &      & Answer format: Solution1/Solution2 \\
                            & assistant & The correct answer is [ANSWER] \\
                \midrule
                SIQA       & system & Please choose the correct answer to the question \\
						   & & based on the context provided. \\
                            &      & Follow the answer format, full answer not needed. \\
                            & user & Context:  [CONTEXT] \\
                            &      & Question:  [QUESTION] \\
                            &      & A: [ANSWER\_A] \\
                            &      & B: [ANSWER\_B] \\
                            &      & C: [ANSWER\_C] \\
                            &      & Answer format: A/B/C \\
                            & assistant & The correct answer is [ANSWER] \\
                \midrule
                hellaswag       & system & Please choose the correct ending to complete the given sentence. \\
                            &      & Follow the answer format, full answer not needed. \\
                            & user & [ACTIVITY\_lABEL]: [CONTEXT] \\
                            &      & Ending1: [ENDING\_1] \\
                            &      & Ending2: [ENDING\_2] \\
                            &      & Ending3: [ENDING\_3] \\
                            &      & Ending4: [ENDING\_4] \\
                            &      & Answer format: Ending1/Ending2/Ending3/Ending4 \\
                            & assistant & The correct answer is [ANSWER] \\
                \midrule
                winogrande       & system & Please choose the correct answer to fill \\
								 & & in the blank to complete the given sentence. \\
                            &      & Follow the answer format, full answer not needed. \\
                            & user & Sentence:  [SENTENCE] \\
                            &      & Option1: [OPTION\_1] \\
                            &      & Option2: [OPTION\_2] \\
                            &      & Answer format: Option1/Option2 \\
                            & assistant & The correct answer is [ANSWER] \\
                \midrule
                ARC-e \& ARC-c       & system & Please choose the correct answer to the question. \\
                            &      & Follow the answer format, full answer not needed. \\
                            & user & Question:  [QUESTION] \\
                            &      & Answer1: [ANSWER\_1] \\
                            &      & Answer2: [ANSWER\_2] \\
                            &      & Answer3: [ANSWER\_3] \\
                            &      & Answer4: [ANSWER\_4] \\
                            &      & Answer format: Answer1/Answer2/Answer3/Answer4 \\
                            & assistant & The correct answer is [ANSWER] \\
                \midrule
                OBQA       & system & Please choose the correct answer to the question. \\
                            &      & Follow the answer format, full answer not needed. \\
                            & user & Question:  [QUESTION] \\
                            &      & Answer1: [ANSWER\_1] \\
                            &      & Answer2: [ANSWER\_2] \\
                            &      & Answer3: [ANSWER\_3] \\
                            &      & Answer4: [ANSWER\_4] \\
                            &      & Answer format: Answer1/Answer2/Answer3/Answer4 \\
                            & assistant & The correct answer is [ANSWER] \\
			\bottomrule
		\end{tabular}
		\label{table:csr_prompts}
	\end{table}

\section{Hyperparameters}%
\label{sec:hyperparameter}
For each optimization method, we carefully preselected the learning rate (step size). 
For the Riemannion method, we set the momentum to $0.9$, and for the Muon method, we set the momentum to $0.95$. 
In both cases, we additionally tuned the weight-decay hyperparameter (see Algorithm~\ref{alg:fr_riemannion}). 
The final choices of these hyperparameters are summarized in Table~\ref{table:adam_commonsense_hyper}.

In the Commonsense reasoning benchmark, the hyperparameters are reported in Table \ref{table:sgd_commonsense_hyper} for the SGD-like methods and in Table \ref{table:adam_commonsense_hyper} for the Adam-like methods. 
For the LOI method, we selected the following hyperparameters: 
the backprop RSVD oversampling parameter was set to $r_k = 16$, 
the backprop powerstep parameter was set to $q = 1$, 
and the multiplier $\alpha$ was set to $\frac{-0.01}{\sqrt{r}}$.

    \begin{table}[h!tp]
    \centering
    \caption{The parameters for different Adam variants for fine-tuning on the Commonsense reasoning dataset}
    \begin{tabular}{lcc}
        \toprule
        Optimizer &  learning rate & weight decay\\
        \midrule
            Adam & $0.0002$ & 0.01 \\
            DoRA & $0.0003$ & 0.01 \\
            Muon & $0.0005$ & 0.2 \\
            DoneRITE & $0.0005$ & 0.0001 \\
            RPrecAdamW & $0.0005$ & 0.5 \\
            Riemannion & $0.0001$ & 0.00316 \\ 
        \bottomrule
    \end{tabular}
    \label{table:adam_commonsense_hyper}
    \end{table}

    \begin{table}[h!tp]
    \caption{The parameters for different SGD variants (RiemannLoRA with \texttt{simulate\_Adam} flag disabled) for fine-tuning on the Commonsense reasoning dataset}
    \centering
    \begin{tabular}{lc}
        \toprule
        Optimizer &  learning rate\\
        \midrule
            LoRA & $0.1$ \\
            LoRA-LOI & $0.06$ \\
            RSLoRA & $0.1$ \\
            RiemannLoRA & $0.1$ \\
            RiemannLoRA-LOI & $0.07$ \\
        \bottomrule
    \end{tabular}
    \label{table:sgd_commonsense_hyper}
	\end{table}

	The non-tuned hyperparameters used for experiments on the Commonsense Reasoning dataset are presented in \ref{table:rest_nonselected_hyper}.
    \begin{table}[h!tp]
    \centering
    \caption{Other non-tuned hyperparameter configurations for experiments on Commonsense reasoning dataset}
    \begin{tabular}{lc}
    \toprule
    Dataset & Commonsense reasoning \\
    Hyperparameters & \\
    \midrule
    Rank $r$ & $16$ \\
    Dropout & $0.05$ \\
    LR Scheduler & Linear \\
    Batch size & $64$ \\
    Epochs & $2$ \\
    Warmup ratio & $0.1$ \\
    \bottomrule
\end{tabular}
\label{table:rest_nonselected_hyper}
\end{table}

\section{Subject-driven generation}%
\label{sec:sdg}

	\paragraph{Training details} We used Stable Diffusion-2-base model with a batch size of $4$ for all experiments. For both methods, we set the betas to $0.9$ and $0.999$ and the weight decay to $0.1$.  In all variations of our approach, we set $q = 15$, $p$ = rank and 

    $\alpha = -\frac{1}{\sqrt{\text{rank}}}$. We used a learning rate of 2e-5 to train both our and the LoRA models, which were used to generate  images shown in Figures \ref{fig:sdg_visual} and \ref{fig:sdg_vis_ablat}.

\paragraph{Evaluation details} We used the DreamBooth dataset, which contains 25 different prompts and 30 various concepts. Due to computational costs, we only used half of the proposed concepts to evaluate metrics: \texttt{can}, \texttt{candle}, \texttt{cat}, \texttt{cat2}, \texttt{colorful\_sneaker}, \texttt{dog2}, \texttt{dog3}, \texttt{dog5}, \texttt{dog6}, \texttt{dog7}, \texttt{dog8}, \texttt{fancy\_boot}, \texttt{grey\_sloth\_plushie}, \texttt{pink\_sunglasses}, \texttt{vase}. To measure the similarity between the original concept and the generated images, we used Image Similarity (IS): we synthesized 30 images for the base prompt <<a photo of a V*>> and 10 images for each of 25 editing prompts (like <<a V* on top of a dirt road>>). We then measured the average pairwise cosine similarity with reference photos of the concept using the DINO model. To check the correspondence between the generated images and the textual prompts, we calculated Text Similarity (TS): we evaluated each concept with each of 25 prompts, synthesizing 10 images per prompt, and calculating the average pairwise cosine similarity with the prompts using the CLIP ViTB/32 model.

    \begin{figure}
        \centering
        \includegraphics[width=1\linewidth]{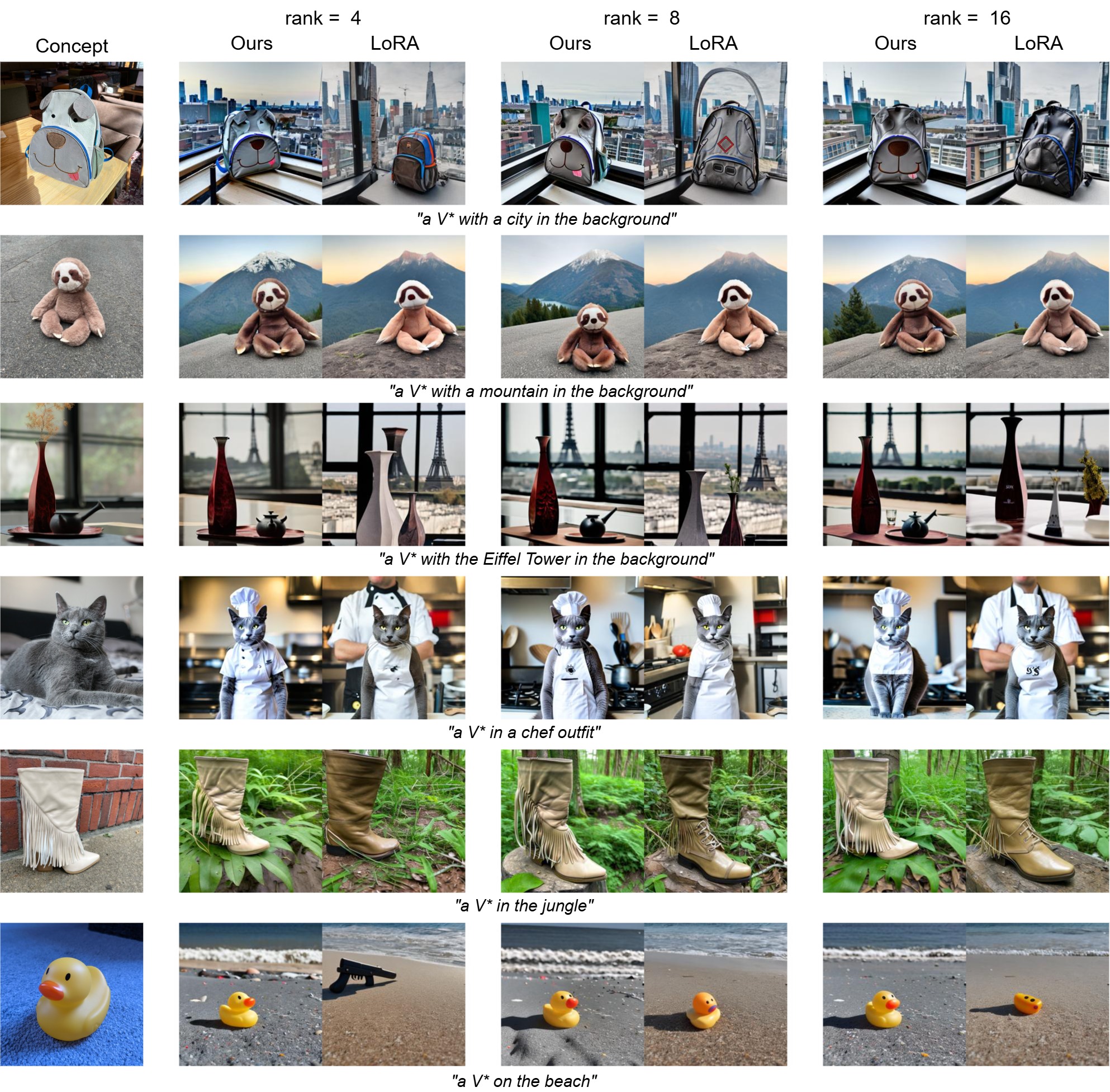}
        \caption{Additional visual comparison of our method and LoRA, checkpoint $600$}
        \label{fig:sdg_vis_ablat}
    \end{figure}

    \paragraph{Additional results} Figure \ref{fig:sdg_vis_ablat} shows an additional visual comparison of our method and LoRA after 600 training steps. As can be seen, our method learns the concept much faster than the original LoRA, while preserving editing capabilities.

\section{Additional experimental results}%
\label{sec:additional_experiment_results}

\subsection{Ablation study on initialization}\label{sub:study_initialization}

The fine-tuning optimization was carried out by AdamW (\citet{loshchilov2017decoupled})
First of the first, for every approach tested  
we preselected a suitable optimization 
step sizes. The list of all selected hyper-parameters for each method is 
specified in the Appendix \ref{sec:hyperparameter}. Table \ref{table:init_dependence} contains the accuracy 
of the trained model's responses on the test dataset. The notation <<LoRA-A>> 
in the table means the utilization of the vanilla LoRA with non-zero $A$, 
the notation <<LoRA-B>> means the same with non-zero $B$, the notations
<<Stiefel-A>> and <<Stiefel-B>> indicate 
vanilla LoRA with orthonormal initialization (for matrix $A$ and $B$ respectively). The naming of <<Stiefel-both>> involves taking the 
factors $A, B$ with orthonormal columns and with initialization like in (\ref{eq:intro_splitting}).

\begin{table}[h!tp]
    \centering
    \caption{The average accuracy among $8$ tasks of fine-tuned Llama 3.2-1b via Adam
	using different LoRA initialization variants, tested on 
    Commonsense reasoning benchmark. <<LOI>> stands for locally optimal initialization. 
    LoRA rank is set to 16.}
    \setlength{\tabcolsep}{3pt}
    \begin{tabular}{lrrrrrrrrr}
        \toprule
        Task & BoolQ & PIQA & SIQA & hella-& wino-& ARC- & ARC- & OBQA & All \\
        Initialization &  &  &  &  swag &  grande &  E&  C& &  \\
        \midrule
        Raw & $40.1$ & $55.4$ & $50.3$ & $25.8$ & $50.0$ & $61.9$ & $41.8$ & $42.8$ & $46.0$ \\
        LoRA-A & $66.2$ & $80.2$ & $76.0$ & $87.0$ & $65.1$ & $78.4$ & $63.6$ & $76.3$ & $74.1$ \\
        LoRA-B & $65.9$ & $77.9$ & $74.2$ & $82.9$ & $61.2$ & $73.9$ & $60.6$ & $72.2$ & $71.1$ \\
        Pissa & $66.4$ & $80.1$ & $75.6$ & $87.6$ & $63.1$ & $77.6$ & $64.3$ & $75.0$ & $73.7$ \\
        Stiefel-A & $65.4$ & $79.5$ & $74.9$ & $87.2$ & $62.4$ & $79.3$ & $62.3$ & $75.5$ & $73.3$ \\
        Stiefel-B & $66.4$ & $80.6$ & $76.0$ & $87.5$ & $64.3$ & $78.5$ & $64.9$ & $74.6$ & $74.1$ \\
        Stiefel-both & $66.3$ & $79.7$ & $75.8$ & $86.4$ & $64.0$ & $78.3$ & $62.4$ & $73.9$ & $73.4$ \\
        LOI & $65.6$ & $81.3$ & $75.7$ & $87.7$ & $65.7$ & $77.9$ & $65.0$ & $75.2$ & $74.3$ \\
        LoRA-GA & $65.4$ & $77.1$ & $74.9$ & $84.5$ & $58.3$ & $75.4$ & $61.9$ & $72.2$ & $71.2$ \\
        \bottomrule
    \end{tabular}
    \label{table:init_dependence}
\end{table}

\section{Computational cost}%
\label{sec:overhead}

To measure execution time, we used a virtual server instantiated on a compute node equipped with an \texttt{Intel Xeon Gold 6240R} CPU and \texttt{Nvidia Tesla V100} GPU. The virtual machine was provisioned with 8 CPU cores, 16 GB of RAM and a single Tesla V100 GPU. Experiments were executed using the Hugging Face \texttt{transformers} library; timing corresponds to the execution time of the \texttt{trainer.train()} call. Measurements were conducted according to the following procedure. For each combination of method, LoRA adapter rank, and batch size, four timing runs were performed and the minimum execution time was selected. The measured quantity was the time spent on 16 optimizer steps. Time spent on initialization was excluded. For each combination of rank and batch size, execution times for each of the methods were measured sequentially. In the present work, rank refers to the rank of the LoRA adapter, and method denotes one of two optimization schemes: Adam (baseline) and Riemannion (proposed method).

Figure \ref{fig:overhead_8b} shows the relative increase in per-step execution time of the proposed Riemannion method compared with Adam, computed as
\begin{figure}[h!tp]
    \centering
    \includegraphics[width=0.6\linewidth]{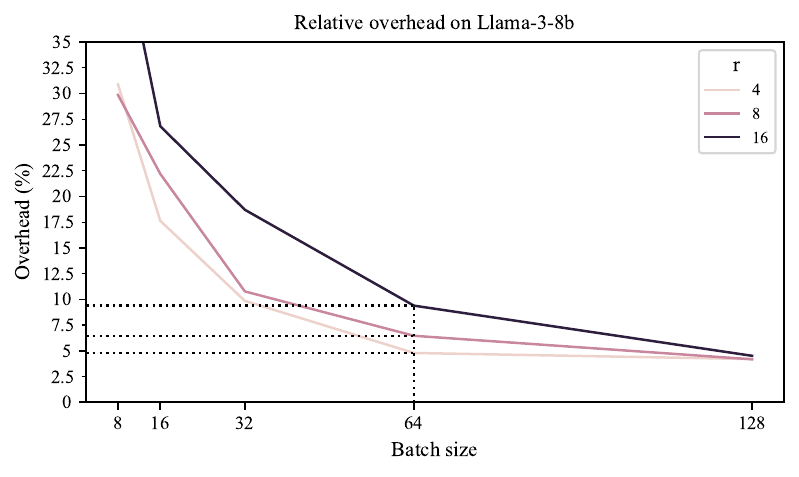}
    \caption{Relative Time Cost $\left( \text{calculated as } \left(T_\text{Riemannion} - T_\text{Adam}\right) \big / {T_\text{Adam}} \right)$ of Riemannion vs. Adam during Llama 3-8B fine-tuning, as a function of LoRA rank and batch size.}
    \label{fig:overhead_8b}
\end{figure}

\section{LLM usage}
We used an LLM only for minor language polishing to improve readability and grammar; it was not involved in research ideation, methodology, analysis, or substantive writing, and all research ideas and arguments were developed entirely by the authors.

\end{document}